\documentclass[10pt,twocolumn,letterpaper]{article}

\usepackage[pagenumbers]{cvpr} 

\usepackage{amsthm}
\usepackage{tikz}
\usepackage{graphicx}
\usepackage{readarray}
\usepackage{pgfplots}
\usepackage{pgfplotstable}
\pgfplotsset{compat=1.18}
\usepackage{bm}
\usetikzlibrary{shapes.misc, positioning, calc, arrows.meta}
\usepackage{colortbl}
\usepackage{float}
\usepackage{listings}

\definecolor{codegreen}{rgb}{0,0.6,0}
\definecolor{codegray}{rgb}{0.5,0.5,0.5}
\definecolor{codepurple}{rgb}{0.58,0,0.82}
\definecolor{backcolour}{rgb}{0.95,0.95,0.92}
\lstdefinestyle{mystyle}{
    backgroundcolor=\color{backcolour},   
    commentstyle=\color{codegreen},
    keywordstyle=\color{magenta},
    numberstyle=\tiny\color{codegray},
    stringstyle=\color{codepurple},
    basicstyle=\ttfamily\footnotesize,
    breakatwhitespace=false,         
    breaklines=true,                 
    captionpos=b,                    
    keepspaces=true,                 
    numbers=left,                    
    numbersep=5pt,                  
    showspaces=false,                
    showstringspaces=false,
    showtabs=false,                  
    tabsize=2
}

\definecolor{cvprblue}{rgb}{0.21,0.49,0.74}
\usepackage[pagebackref,breaklinks,colorlinks,allcolors=cvprblue]{hyperref}

\newtheorem{proposition}{Proposition}
\newtheorem{theorem}[proposition]{Theorem}
\newtheorem{lemma}[proposition]{Lemma}
\DeclareMathOperator{\PoM}{\text{PoM}}

\title{\includegraphics[height=1.3\fontcharht\font`f]{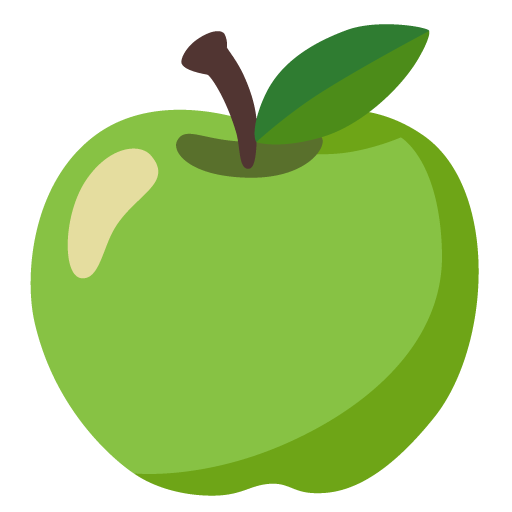}PoM: Efficient Image and Video Generation with the Polynomial Mixer}

\author{David Picard$^1$, Nicolas Dufour$^{1,2}$\\
$^1$LIGM, École Nationale des Ponts et Chaussées, IP Paris, Univ Gustave Eiffel, CNRS, France\\
$^2$LIX, École Polytechnique, IP Paris, CNRS, France\\
{\tt\small \{david.picard,nicolas.dufour\}@enpc.fr}
}

\begin{document}

\maketitle
\begin{abstract}
Diffusion models based on Multi-Head Attention (MHA) have become ubiquitous to generate high quality images and videos. However, encoding an image or a video as a sequence of patches results in costly attention patterns, as the requirements both in terms of memory and compute grow quadratically. To alleviate this problem, we propose a drop-in replacement for MHA called the Polynomial Mixer (PoM) that has the benefit of encoding the entire sequence into an explicit state. PoM has a linear complexity with respect to the number of tokens. This explicit state also allows us to generate frames in a sequential fashion, minimizing memory and compute requirement, while still being able to train in parallel. We show the Polynomial Mixer is a universal sequence-to-sequence approximator, just like regular MHA. We adapt several Diffusion Transformers (DiT) for generating images and videos with PoM replacing MHA, and we obtain high quality samples while using less computational resources. The code is available at \url{https://github.com/davidpicard/HoMM}.

\end{abstract}

\section{Introduction}
\begin{figure}
    \centering
    \resizebox{\columnwidth}{!}{
\begin{tikzpicture}
    \begin{axis}[
        width=\columnwidth,
        height=0.8\columnwidth,
        xlabel={Image resolution},
        ylabel={Time in second/image},
        ymode=log,
        xtick={256, 1024, 2048, 4096},
        grid=both,
        grid style={line width=.1pt, draw=gray!10},
        major grid style={line width=.2pt, draw=gray!30},
        xmin=0, xmax=4352,
        ymin=0.01, ymax=6,
        legend style={at={(0.98,0.02)}, anchor=south east, font=\footnotesize},
        legend cell align=left,
        every axis plot/.append style={very thick}
    ]

    \addplot[color=blue!70, mark=o] coordinates {
        (256, 0.038862082336563616)
        (384, 0.03905078684678301)
        (512, 0.04221228306181729)
        (768, 0.05795255176257342)
        (1024, 0.09442358191823587)
        (1536, 0.17702449211617932)
        (2048, 0.29645248390734197)
        (3072, 0.6015695175528526)
        (4096, 1.0516053920192645)
    };
    \addlegendentry{PoM forward+backward}

    \addplot[color=blue!70, mark=o, dashed, mark options={solid}] coordinates {
        (256, 0.015069471425376832)
        (384, 0.015871938960626723)
        (512, 0.0189380434085615)
        (768, 0.02909953865222633)
        (1024, 0.04561432354385033)
        (1536, 0.080849758409895)
        (2048, 0.13602926889434458)
        (3072, 0.2696375051792711)
        (4096, 0.47321898558177056)

    };
    \addlegendentry{PoM forward}

    \addplot[color=orange!80, mark=square] coordinates {
        (256, 0.03005698923021555)
        (384, 0.032288018770050254)
        (512, 0.03444906664080918)
        (768, 0.05218656954821199)
        (1024, 0.08344621217343956)
        (1536, 0.21130420179106296)
        (2048, 0.4523333204118535)
        (3072, 1.639214531343896)
        (4096, 4.577012917282991)
    };
    \addlegendentry{MHA forward+backward}

    \addplot[color=orange!80, mark=square, dashed, mark options={solid}] coordinates {
        (256, 0.010657461867667735)
        (384, 0.012714778319932521)
        (512, 0.014707973287440836)
        (768, 0.02519212323939428)
        (1024, 0.037897837879136205)
        (1536, 0.08845603726571426)
        (2048, 0.18336747443769127)
        (3072, 0.6122184240771458)
        (4096, 1.6892871374706737)
    };
    \addlegendentry{MHA forward}
    \end{axis}
\end{tikzpicture}
    }
    \caption{\textbf{Comparison between the speed of PoM and Multi-Head Attention (MHA) in the same DiT-XL/2 architecture for different image resolutions.} We use an H100 GPU and compute the average time on 100 synthetic training batches to perform the forward or forward+backward passes. We use synthetic data to remove the influence from data loading. Training with PoM is less costly than inference with MHA at higher resolutions.}
    \label{fig:enter-label}
\end{figure}
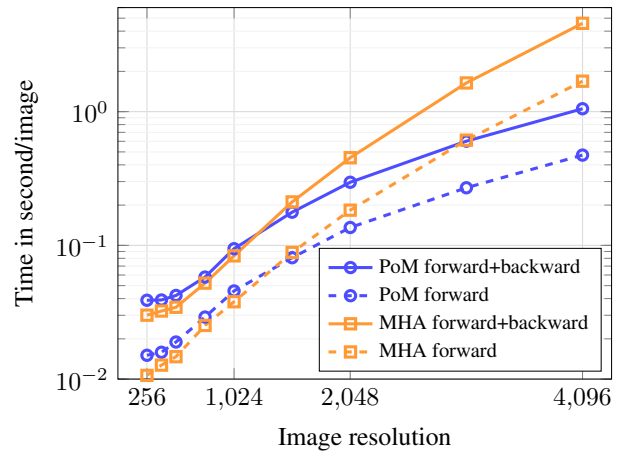

In a sudden change of pace, high quality image and video generation have evolved from a task seemingly impossible to achieve to a task almost solved by available commercial or open-source tools like Stable Diffusion 3~\cite{esser24icml}, Sora~\cite{sora} or MovieGen~\cite{moviegen}.
At the heart of this success lies the Multi-head Attention (MHA) in the transformer architecture~\cite{vaswani17nips} that has excellent scaling properties~\cite{zhai22cvpr,peebles23iccv}.
These so-called scaling laws~\cite{kaplan2020scaling} enable \emph{brute-forcing} complex problems such as image and video generation by using very large models trained on gigantic data, at the expense of an ever increasing computational cost.
The main focus of current research lies thus in scaling transformer-based approaches to larger models handling larger datasets.

The issue with transformers is that the computational cost increases quadratically with the sequence length due to the pairwise computation in MHA.
This means that generating an image at twice the spatial resolution (respectively a video at twice the resolution and double the duration) results in 4 times more patches and thus 16 times more computational cost (respectively 8 times more patches and thus 64 times more computational cost).
Attempts at having transformers with sub-quadratic complexity~\cite{child2019generating,kitaev2020reformer,wang2020linformer} introduce the additional constraint of fixing the number of tokens, which prevents generating images or videos of different sizes.
Alternatively, recurrent models such as State-Space Models (SSM)~\cite{gu21iclr,gu21nips} have been investigated for the task~\cite{yan24cvpr,teng2024dim,hu24eccv} since their complexity is linear with the sequence length~\cite{gu24colm}.
However, they introduce an arbitrary causal raster scan of the sequence that does not fit the 2D geometry of images very well.

In this paper, we enable better scaling in large generative models by introducing a new building block called the Polynomial Mixer (PoM).
PoM has a linear complexity like SSMs while still enabling all pairwise information to be processed like in MHA, obtaining effectively the best of both worlds.
From a theoretical standpoint, we prove PoM can be used as a drop-in replacement for attention.
Doing so in the popular DiT architecture~\cite{peebles23iccv,ma24eccv} results in improved scaling such that at higher resolutions, it becomes less costly to train a model with PoM than to perform inference with a model using MHA, as shown on Figure~\ref{fig:scaling_diffusion}.

To sum up, the contributions of this paper are the following:
\begin{itemize}
    \item[$\checkmark$] We introduce the Polynomial Mixer (PoM), a replacement for MHA that has a linear complexity with respect to the sequence length and without sacrificing generation quality;
    \item[$\checkmark$] We prove that models equipped with PoM are universal sequence-to-sequence approximators;
    \item[$\checkmark$] We train DiT-inspired image generative models and obtain results of similar quality while being much more compute efficient at higher resolutions;
    \item[$\checkmark$] We train video generative models leveraging PoM with a constant processing cost per frame while not sacrificing on visual quality.
\end{itemize}

Our contribution is therefore primarily fundamental: We show that it is possible to
train generative models with an alternative mecanism to MHA. We believe this direction will not only ground future research on high resolution images and very long videos generation, but also
could benefit many areas of research (\textit{e.g.}, large language models, vision-language models, etc). 

\section{Related Work}

\paragraph{Diffusion}
Diffusion models~\cite{ho20nips,nichol21icml,song21iclr} learn a neural operator that produces natural images from noise using a forward-reverse set of processes.
The forward process consists in pushing the distribution of natural images forward to a known distribution, typically Gaussian, which can be done by adding increasing level of noise to the image.
The reverse process does not have an explicit solution, but can be approximated by a neural network by regressing the local inverse of the forward process, \textit{i.e.}, solving
\begin{align}
    &\min_\theta \mathbb{E}_{t\sim \mathcal{U}(0,1)}\left[\|\varepsilon_t - f_\theta(x_t, t)\|^2\right],\\
    &\text{ s.t. } x_t = \alpha_t x_0 + \gamma_t \varepsilon_t,\, \varepsilon_t \sim \mathcal{N}(0,1).
\end{align}
Here, $\alpha_t$ and $\gamma_t$ are chosen such that $x_0$ corresponds to a natural image whereas $x_1$ corresponds to pure Gaussian noise.
A great amount of research has been put into finding better noise schedules ($\alpha_t$ and $\gamma_t$)~\cite{balaji22eDiffI,karras24cvpr,hang2024improved}, or improving the quantity that is regressed~\cite{lipman22iclr,shi24nips,liu23iclr}, keeping the general idea of learning to invert step by step the stochastic differential equation that transforms an image into noise.

For image and generation, most efforts have been poured into designing efficient architectures at the task. While the original DDPM papers~\cite{ho20nips,nichol21icml} sample images in pixel space, making it unsuitable for large resolution, the most groundbreaking improvement was introduced by Stable Diffusion~\cite{rombach22cvpr} with the addition of a variational auto-encoder (VAE) that allows the diffusion process to be performed in a lower dimensional latent space.
Stable Diffusion uses a U-Net architecture complemented by attention layers~\cite{rombach22cvpr,Si_2024_CVPR}.
To benefit more from the scaling properties of transformers~\cite{kaplan2020scaling,zhai22cvpr}, simpler approaches based solely on transformer layers has been proposed in DiT~\cite{peebles23iccv} and the subsequent flow-matching version SiT~\cite{ma24eccv}.
Most modern text-to-image generation models are now based on Transformer layers rather than the U-Net~\cite{hatamizadeh2025diffit,esser24icml,chen2024pixart,gao2024lumina}. \cite{crowson2024scalable, gu2023matryoshka}, train efficient pixel space transformers models by leveraging multiscale training and SwinAttention.
Similarly, RIN~\cite{jabri23icml, chen2023fit} also proposes an approach using attention only, albeit in a Perceiver-IO~\cite{jaegle22iclr} inspired architecture that uses cross-attention to perform most of the computation in a smaller latent space, and has been successfully extended to text-to-image~\cite{dufour24cvpr}.
In addition to architectures and sampling~\cite{Zhou_2024_CVPR,Bai_2024_CVPR,zhao2023mobilediffusion}, the importance of training is also highlighted in recent works, from resampling the training data~\cite{Gokaslan_2024_CVPR,Liu_2024_CVPR} to RL~\cite{Wallace_2024_CVPR,wei2024powerful,lee2025parrot} and model averaging~\cite{Karras_2024_CVPR}.

In video generation~\cite{villegas22iclr,ho2022video,zhao2025magdiff,singer23make,gupta2025photorealistic}, early attempts have focused on extending existing text-to-image models to benefit from their large scale pretraining~\cite{rombach23cvpr,kwon2024harivo,ge2023preserve,ho2022imagen,hong2023cogvideo,girdhar24factorizing}. 
However, the drawback of such approaches is that they re-use the VAE of existing text-to-image models which does not encode temporal information, which is thus not compressed.
As such, novel architectures using a 2D+t VAE such as CogVideoX~\cite{yang2024cogvideox}, PyramidFlow~\cite{jin2024pyramidal} can benefit from a smaller latent space leading to less computational costs.

\paragraph{Fast alternative to attention}
Since the introduction of Transformers~\cite{vaswani17nips}, many effort have been made to reduce the quadratic complexity of MHA~\cite{child2019generating,kitaev2020reformer,wang2020linformer}.
Notably, methods like Reformer~\cite{kitaev2020reformer} use fast approximate neighbors to reduce the size of the attention matrix based on the assumption that most tokens will have zero attention. 
To go further, Linformer~\cite{wang2020linformer} proposes to compute an explicit low rank projection of the keys and the values to reduce the complexity of MHA for each query from the size of the sequence $n$ to an arbitrary chosen number $k \ll n$.
The main drawback of such approach is that $n$ and $k$ are fixed, which means that the model can no longer process sequences of varying length.
With the advent of Large Language Models and their ability to process extremely long sequences~\cite{achiam2023gpt,team2023gemini,dubey2024llama}, recent efforts have been put on more efficient implementations such as Flash-Attention~\cite{dao22nips,dao2023flashattention} or KV-cache~\cite{brandon2024reducing,luohe2024keep} which seem sufficient for text.
However for visual content, the sequence length grows quadratically with the resolution, which, because MHA is also quadratic in the number of tokens, leads to quartic computational and memory complexity.

Alternatively, some attempts have been made to just remove the Multi-Head Attention, such as in Mlp-Mixer~\cite{tolstikhin21nips} and Resmlp~\cite{touvron2022resmlp} that replace MHA with simple projection on the transpose tensor (\textit{i.e.}, considering the sequence dimension as the features).
These approaches have been shown to obtain competitive results, but similarly to Linformer, they imply a fix sequence length since this length is now an intrinsic dimension of the projection in the transpose direction.
More recently, State-Space Models (SSM)~\cite{gu21iclr,gu21nips} have become the focus of recent work especially in language modeling~\cite{zuo2024falcon,dao24icml,lieber2024jamba,glorioso2024zamba}.
SSM are recurrent models, which is highly beneficial for language modeling because of the causal property of text. In that case, the complexity to generate the next token becomes constant.
In visual content however, there is no such natural causality pattern in the spatial dimensions.
Attempt to use such models for vision tasks have been successful~\cite{zhu2024vision,liu2024vmambavisualstatespace,pei2024efficientvmambaatrousselectivescan}, albeit at the cost of enforcing an arbitrary 1-dimensional scan order of the tokens that does not encode well the 2D nature of an image.
In image generation using diffusion~\cite{hu24eccv,yan24cvpr}, since the model has to be iterated, this results in a doubly sequential processing (space and iterations) that does not benefit from the parallel nature of processing images.
For video however, the causal aspect is natural over the time dimension, and recurrent approaches may be more efficient.

\section{Polynomial Mixer and Polymorpher}
We define a \emph{Polymorpher} block as a sequence-to-sequence function mapping $\mathbb{R}^{d\times n}$ to $\mathbb{R}^{d\times n}$, composed of two residual blocks, a \emph{Polynomial Mixer} and a feed-forward block. 

For a sequence $X\in \mathbb{R}^{d\times n}$, the Polynomial Mixer ($\PoM$) shown on Figure \ref{fig:pom} is defined as follows:
\begin{align}
    \PoM(X) &= W_o \left[\sigma(W_s X) \circ H(X)\bm{1}^\top\right],\text{ with}\\
    H(X) &= \left[h(W_1 X); \dots; \prod_{m=1}^kh(W_mX)\right]\bm{1}, 
\end{align}
where $k$ is the degree of the Polynomial Mixer, $\sigma$ is the sigmoid function, $h$ an activation function, $\circ$ and $\prod$ the element-wise (Hadamard) product, and $\bm{1}$ a vector of the appropriate dimension filled with ones. The notation $[\cdot; \cdot]$ is for vertical concatenation. The matrices $W_o\in \mathbb{R}^{d\times kD}$, $W_s\in \mathbb{R}^{kD\times d}$ and $W_1, \dots, W_k \in \mathbb{R}^{D\times d}$ are the learnable parameters of the Polynomial Mixer.

\begin{figure}
    \centering
    \resizebox{\columnwidth}{!}{
    \begin{tikzpicture}
\node[draw=white, rounded corners, minimum width=2.5cm, minimum height=0.7cm, fill=gray!20, label=above:$X$] (input) {};

\foreach \x in {0.35, 0.95, 1.55, 2.15} {
    \node[draw=gray, fill=white, rounded corners, minimum width=0.5cm, minimum height=0.5cm] at ([xshift=\x cm]input.west) {};
}

\node[draw=white, rounded corners, minimum width=2.5cm, minimum height=2.5cm, fill=orange!20, right=1.5cm of input, yshift=1.5cm, label={above:$\left[h(W_1 X); \dots; \prod_m^k h(W_m X)\right]\bm{1}$}] (topbox) {};

\foreach \x in {0.35, 0.95, 1.55, 2.15} {
    \node[draw=orange, fill=white, rounded corners, minimum width=0.5cm, minimum height=2.3cm] (rect) at ([xshift=\x cm]topbox.west) {};
    \draw[orange] ($(rect.south west)!0.33!(rect.north west)$) -- ($(rect.south east)!0.33!(rect.north east)$);
    \draw[orange] ($(rect.south west)!0.66!(rect.north west)$) -- ($(rect.south east)!0.66!(rect.north east)$);
}

\node[draw=white, rounded corners, minimum width=0.7cm, minimum height=2.5cm, fill=orange!20, right=1.5cm of topbox, yshift=0cm, label={above:$H(X)$}] (toprightbox) {};

\node[draw=orange!60, fill=white, rounded corners, minimum width=0.5cm, minimum height=2.3cm] (rect) at ([xshift=0.35 cm]toprightbox.west) {};
\draw[orange] ($(rect.south west)!0.33!(rect.north west)$) -- ($(rect.south east)!0.33!(rect.north east)$);
\draw[orange] ($(rect.south west)!0.66!(rect.north west)$) -- ($(rect.south east)!0.66!(rect.north east)$);

\node[draw=white, rounded corners, minimum width=2.5cm, minimum height=2.5cm, fill=blue!15, right=2.5cm of input, yshift=-1.5cm, label={above:$S(X) = \sigma(W_s X)$}] (bottombox) {};

\foreach \x in {0.35, 0.95, 1.55, 2.15} {
    \node[draw=blue!60, fill=white, rounded corners, minimum width=0.5cm, minimum height=2.3cm] (rect) at ([xshift=\x cm]bottombox.west) {};
}

\node[draw=white, rounded corners, minimum width=2.5cm, minimum height=2.5cm, fill=green!20, right=1cm of toprightbox, yshift=-1.5cm, label={above:$S(X) \circ H(X)\bm{1}^\top$}] (rightbox) {};

\foreach \x in {0.35, 0.95, 1.55, 2.15} {
    \node[draw=green!60, fill=white, rounded corners, minimum width=0.5cm, minimum height=2.3cm] (rect) at ([xshift=\x cm]rightbox.west) {};
}

\node[draw=white, rounded corners, minimum width=2.5cm, minimum height=0.7cm, fill=gray!20, right=1cm of rightbox, label=above:$W_o Z$] (output) {};

\foreach \x in {0.35, 0.95, 1.55, 2.15} {
    \node[draw=gray, fill=white, rounded corners, minimum width=0.5cm, minimum height=0.5cm] at ([xshift=\x cm]output.west) {};
}

\draw[->, thin, -{Stealth[length=2mm,width=1.4mm]}] (input.east) to[out=0, in=180]  (topbox.west);
\draw[->, thin, -{Stealth[length=2mm,width=1.4mm]}] (topbox.east) to[out=0, in=180] node[above, midway] {$\sum$} (toprightbox.west);
\draw[->, thin, -{Stealth[length=2mm,width=1.4mm]}] (toprightbox.east) to[out=0, in=180] (rightbox.west);
\draw[->, thin, -{Stealth[length=2mm,width=1.4mm]}] (input.east) to[out=0, in=180] (bottombox.west);
\draw[->, thin, -{Stealth[length=2mm,width=1.4mm]}] (bottombox.east) to[out=0, in=180] (rightbox.west);
\draw[->, thin, -{Stealth[length=2mm,width=1.4mm]}] (rightbox.east) to[out=0, in=180] (output.west);
\end{tikzpicture}
    }
    \caption{\textbf{Diagram for the Polynomial Mixer.} The input sequence is split into two paths. The top path expands each token using a polynomial before they are mixed (averaged)² into a single representation. The bottom path expands the tokens into gating coefficients. Both paths are recombined and projected back into the input dimension.}
    \label{fig:pom}
\end{figure}
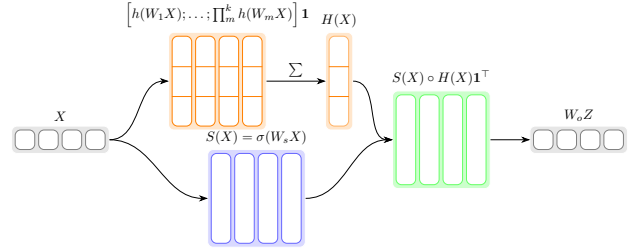

The idea of the Polynomial Mixer is to that the sequence $X\in \mathbb{R}^{d\times n}$ is uniquely summarized into the representation $H(X) \in \mathbb{R}^{kD\times 1}$. Each element in $X$ then gets to query $H(X)$ independently thanks to the map $S(W_s X) \in \mathbb{R}^{kD\times n}$. The queried information is then projected back into the original space with $W_o$.

Contrarily to MHA that computes all pairwise exchanges of information between tokens in the sequence, the Polynomial Mixer follows a state-representation ($H(X)$) approach where all information is shared in a common memory location that all tokens can access. This state-representation is defined by mixing all tokens of the sequence after they are mapped to a high dimensional space by a learned polynomial, hence the name \emph{Polynomial Mixer}, and a similar approach has been successfully used for learning image representation~\cite{jacob2019metric}. The main benefit is that the complexity of the approach is no longer quadratic but linear with the sequence length $n$.

Taking inspiration from transformers with MHA, we define a Polymorpher block $P$ as alternating residual Polynomial Mixers with feed-forward networks as follows:
\begin{align}
    \text{P}(X) = X + \PoM(X) + \text{FF}(X + \PoM(X)),
\end{align}
with $\text{FF}(X)$ being a two-layer feed-forward network.

A Polymorpher is a drop-in replacement for any Transformer-based architecture as it performs the same role of sequence-to-sequence mapping. The main difference is in its parametrization: A Transformer is configured by the number of heads and their dimension in MHA, whereas the Polymorpher is configured by its degree $k$ and the dimension $D$ of each polynomial.

\subsection{Polymorpher for causal sequences}
\label{sec:causal}

A causal sequence can easily be modeled in $\PoM$ by adding a mask $M$ that prevents summing future tokens into the blackboard. This corresponds to the following  definition
\begin{align}
    \PoM(X, M) = W_o \left[\sigma(W_s X) \circ H(X)\right], \\
    H(X) = \left[h(W_1 X); \dots; \prod_{m=1}^kh(W_mX)\right]M^\top. 
\end{align}
Now $H(X)\in \mathbb{R}^{kD\times n}$ and $M\in \{0, 1\}^{n\times n}$ is a binary matrix that defines which pairs of tokens are related. Just like for MHA, a binary matrix defines an attention pattern that can be arbitrarily chosen. 

In the special case of causal sequences, $M$ is a lower triangular matrix.
Moreover, one can express the mixing part of the Polynomial Mixer as an iterative process as follows:
\begin{align}
    H(X)_{:,i} &= \sum_{j\leq i} \left[h(W_1 X); \dots; \prod_{m=1}^kh(W_mX)\right]_{:,j},\\
    &= H(X)_{:,i-1} + \left[h(W_1 X); \dots; \prod_{m=1}^kh(W_mX)\right]_{:,i}.
\end{align}
In this formula, $H(X)_{:,i}$ is an explicit hidden state that is updated by adding the polynomial mapping of the next token.
Such a configuration enables $\mathcal{O}(1)$ inference complexity in the auto-regressive setup, a property that is shared with recurrent networks, but not transformers. Like SSMs, Polymorphers have the best of both worlds, they can train on the whole sequence in parallel and do the inference in the recursive way.

In addition, Polymorphers can handle block causal sequences. Let $M$ be a block causal matrix for some integer block size $K$:
\begin{align}
    M_{i,j} = 1 \text{ if } j \leq \lceil i / K \rceil K \text{ else } 0.
\end{align}
We can now rewrite $H$ as
\begin{align}
\nonumber H(X)_{:,i} = &H(X)_{:,\lfloor i/K\rfloor K} \\
    &+ \sum_{j = \lfloor i/K\rfloor K}^{\lceil i / K \rceil K}\left[h(W_1 X); \dots; \prod_{m=1}^kh(W_mX)\right]_{:,j}.
\end{align}
In this configuration, we can sequentially process groups of tokens at a time during inference, which reduces the memory requirement.
This is in particular practical for video sequences where it makes sense to have a causal mask in the temporal dimension that makes each frame depend on the previous ones, while keeping the ability of all the tokens (patches) of a frame to look at each others, since causality does not have much sense in the spatial dimension.

\subsection{Theoretical analysis}

We first show that $\PoM$ is equivariant, which means that permutations in the input sequence result in permuted outputs. This is a key property that made transformers popular and does not hold for other architectures like convolutions:

\begin{proposition}[Permutation equivariance]
    A Polynomial Mixer is permutation equivariant, i.e., let $X \in \mathbb{R}^{d\times n}$ be a set of vectors and $P$ a column permutation matrix, then $\PoM(XP) = \PoM(X)P$.
\end{proposition}
\begin{proof}
    For a permutation $P$, we have 
    \begin{align}\PoM(XP) = W_o \left[ \sigma(W_s X P) \circ H(X P)\bm{1}^\top \right]. 
    \end{align}
    Notice that $H(X P) = H(X)$ because the sum is permutation invariant, and $\sigma(W_s X P) = \sigma(W_s X) P$ because $\sigma$ is an element-wise operation. 
    Noticing that $H(X)\bm{1}^\top$ has all identical columns allows us to move $P$ outside of the brackets to conclude the proof.
\end{proof}

More importantly, we can also prove a universal approximation theorem for Polymorphers similar to what is well known for Transformers~\cite{Yun20ICLR}. As the polynomial mixer is equivariant, it requires the use of positional encoding, which also underlines the similarity between $\PoM$ and MHA. 

We use the following standard definition of distance between functions that map sequences to sequences. Given two functions $f$ and $g: \mathbb{R}^{d_n}\rightarrow \mathbb{R}^{d_n}$ and an integer $A\leq p\leq \infty$, we define the distance $d_p$ as:
\begin{align}
    d_p(f, g) = \left(\int \| f(X) - g(X)\|_p^pdX\right)^{1/p}.
\end{align}

The following theorem holds:
\begin{theorem}[Universal approximation]
    Let $1 \leq p \leq \infty$ and $\epsilon > 0$, then for any given $f\in \mathcal{F}$ the set of continuous functions that map a compact domain in $\mathbb{R}^{d\times n}$ to $\mathbb{R}^{d\times n}$, there exists a Polymorpher $g$ with learned positional encoding such that $d_p(f, g) \leq \epsilon$.
\end{theorem}

The proof follows exactly the same scheme as in ~\cite{Yun20ICLR}, where most of the heavy lifting is done by the feed-forward networks. Their main argument is to show that MHA can map every token in the sequence to a unique value that depends on the entire sequence, and then the feed-forward blocks can map those unique values to the desired output. In our case, we just have to ensure that the Polynomial Mixer has the same properties as MHA, which is obtained using the following lemma:

\begin{lemma}[Contextual mapping (informal)]
    There exists $k > 0$ for which any Polynomial Mixer $q$ of degree $k$ is a contextual mappings on $\mathbb{R}^{d\times n}$, that is:
    \begin{itemize}
        \item For any $X \in \mathbb{R}^{d\times n}$ with different entries, $q(X)$ has different entries.
        \item For any $X, X' \in \mathbb{R}^{d\times n}$ that differ at least by one element, then all entries of $q(L)$ and $q(L')$ are different.
    \end{itemize}
\end{lemma}

The proof is deferred to the appendix and primarily uses the fact that a sufficiently high degree polynomial is uniquely defined by a sequence of point-wise evaluation. As noted in \cite{Yun20ICLR}, having the contextual mapping property is not so common as it requires to summarize uniquely the context while preserving the identity of the current token.

With these results, we show that a Polymorpher is as potent as a Transformer for sequence modeling.

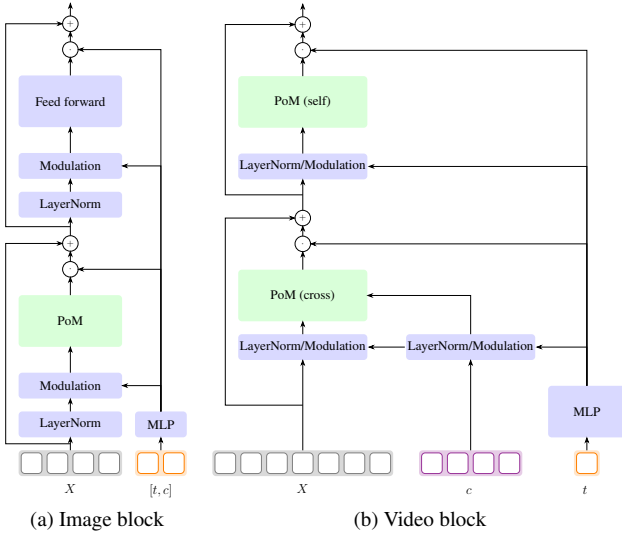
\begin{figure}[tb]
    \centering
    \begin{subfigure}[b]{0.3\columnwidth}
        \centering
        \resizebox{\columnwidth}{!}{
        \begin{tikzpicture}

\begin{scope}[local bounding box=input]
  \fill[gray!30, rounded corners] (0,0) rectangle (4,1);
  \foreach \x in {0.1, 1.1, 2.1, 3.1} {
    \fill[white, rounded corners] (\x, 0.1) rectangle (\x + 0.8, 0.9);
    \draw[gray, rounded corners] (\x, 0.1) rectangle (\x + 0.8, 0.9);
  }
\end{scope}

\begin{scope}[local bounding box=condition]
  \fill[orange!20, rounded corners] (4.5,0) rectangle (6.5,1);
  \foreach \x in {4.6, 5.6} {
    \fill[white, rounded corners] (\x, 0.1) rectangle (\x + 0.8, 0.9);
    \draw[orange, rounded corners] (\x, 0.1) rectangle (\x + 0.8, 0.9);
  }
\end{scope}

\node[below=0.2cm of input] {\Large $X$};
\node[below=0.2cm of condition] {\Large $[t, c]$};

\begin{scope}[local bounding box=mlp]
  \fill[blue!15, rounded corners] (4.5,1.5) rectangle (6.5,2.5);
  \node at (5.5,2) {\Large MLP};
\end{scope}

\draw[->, thin, -{Stealth[length=2mm,width=1.4mm]}] (condition.north) -- (mlp.south);

\begin{scope}[local bounding box=layernorm1]
  \fill[blue!15, rounded corners] (0,1.5) rectangle (4,2.5);
  \node at (2,2) {\Large LayerNorm};
\end{scope}

\draw[->, thin, -{Stealth[length=2mm,width=1.4mm]}] (input.north) -- (layernorm1.south);

\begin{scope}[local bounding box=modulation1]
  \fill[blue!15, rounded corners] (0,3) rectangle (4,4);
  \node at (2,3.5) {\Large Modulation};
\end{scope}

\draw[->, thin, -{Stealth[length=2mm,width=1.4mm]}] (layernorm1.north) -- (modulation1.south);
\draw[->, thin, -{Stealth[length=2mm,width=1.4mm]}] (mlp.north) |- (modulation1.east);

\begin{scope}[local bounding box=pom]
  \fill[green!15, rounded corners] (0,5) rectangle (4,7);
  \node at (2,6) {\Large PoM};
\end{scope}

\draw[->, thin, -{Stealth[length=2mm,width=1.4mm]}] (modulation1.north) -- (pom.south);

\begin{scope}[local bounding box=gate1]
  \draw[thick] (2,8) circle (0.3);
  \node at (2,8) {$\cdot$};
\end{scope}

\draw[->, thin, -{Stealth[length=2mm,width=1.4mm]}] (pom.north) -- (gate1.south);
\draw[->, thin, -{Stealth[length=2mm,width=1.4mm]}] (mlp.north) |- (gate1.east);

\begin{scope}[local bounding box=residual1]
  \draw[thick] (2,9) circle (0.3);
  \node at (2,9) {$+$};
\end{scope}

\draw[->, thin, -{Stealth[length=2mm,width=1.4mm]}] (gate1.north) -- (residual1.south);

\coordinate (midpoint) at ($(input.north)!0.5!(layernorm1.south)$);
\draw[->, thin, -{Stealth[length=2mm,width=1.4mm]}] (midpoint) -- ++(-2.5,0) |- (residual1.west);

\begin{scope}[local bounding box=layernorm2]
  \fill[blue!15, rounded corners] (0,10) rectangle (4,11);
  \node at (2,10.5) {\Large LayerNorm};
\end{scope}

\draw[->, thin, -{Stealth[length=2mm,width=1.4mm]}] (residual1.north) -- (layernorm2.south);

\begin{scope}[local bounding box=modulation2]
  \fill[blue!15, rounded corners] (0,11.5) rectangle (4,12.5);
  \node at (2,12) {\Large Modulation};
\end{scope}

\draw[->, thin, -{Stealth[length=2mm,width=1.4mm]}] (layernorm2.north) -- (modulation2.south);
\draw[->, thin, -{Stealth[length=2mm,width=1.4mm]}] (mlp.north) |- (modulation2.east);

\begin{scope}[local bounding box=ffw]
  \fill[blue!15, rounded corners] (0,13.5) rectangle (4,15.5);
  \node at (2,14.5) {\Large Feed forward};
\end{scope}

\draw[->, thin, -{Stealth[length=2mm,width=1.4mm]}] (modulation2.north) -- (ffw.south);

\begin{scope}[local bounding box=gate2]
  \draw[thick] (2,16.5) circle (0.3);
  \node at (2,16.5) {$\cdot$};
\end{scope}

\draw[->, thin, -{Stealth[length=2mm,width=1.4mm]}] (ffw.north) -- (gate2.south);
\draw[->, thin, -{Stealth[length=2mm,width=1.4mm]}] (mlp.north) |- (gate2.east);

\begin{scope}[local bounding box=residual2]
  \draw[thick] (2,17.5) circle (0.3);
  \node at (2,17.5) {$+$};
\end{scope}

\draw[->, thin, -{Stealth[length=2mm,width=1.4mm]}] (gate2.north) -- (residual2.south);

\coordinate (midpoint) at ($(residual1.north)!0.5!(layernorm2.south)$);
\draw[->, thin, -{Stealth[length=2mm,width=1.4mm]}] (midpoint) -- ++(-2.5,0) |- (residual2.west);

\draw[->, thin, -{Stealth[length=2mm,width=1.4mm]}] (residual2.north) -- ++(0,0.5);

\end{tikzpicture}
        }
        \caption{Image block}
        \label{fig:image_block}
    \end{subfigure}
    \hfill
    \begin{subfigure}[b]{0.67\columnwidth}
        \centering
        \resizebox{\columnwidth}{!}{
        \begin{tikzpicture}

\begin{scope}[local bounding box=input]
  \fill[gray!30, rounded corners] (0,0) rectangle (7,1);
  \foreach \x in {0.1, 1.1, 2.1, 3.1, 4.1, 5.1, 6.1} {
    \fill[white, rounded corners] (\x, 0.1) rectangle (\x + 0.8, 0.9);
    \draw[gray, rounded corners] (\x, 0.1) rectangle (\x + 0.8, 0.9);
  }
\end{scope}

\begin{scope}[local bounding box=condition]
  \fill[violet!20, rounded corners] (8.,0) rectangle (12.,1);
  \foreach \x in {8.1, 9.1, 10.1, 11.1} {
    \fill[white, rounded corners] (\x, 0.1) rectangle (\x + 0.8, 0.9);
    \draw[violet!80, rounded corners] (\x, 0.1) rectangle (\x + 0.8, 0.9);
  }
\end{scope}

\begin{scope}[local bounding box=time]
  \fill[orange!20, rounded corners] (14,0) rectangle (15,1);
  \fill[white, rounded corners] (14.1, 0.1) rectangle (14.9, 0.9);
  \draw[orange, rounded corners] (14.1, 0.1) rectangle (14.9, 0.9);
\end{scope}

\begin{scope}[local bounding box=x_ln]
  \fill[blue!15, rounded corners] (1,4.5) rectangle (6,5.5);
  \node at (3.5,5) {\Large LayerNorm/Modulation};
\end{scope}

\draw[->, thin, -{Stealth[length=2mm,width=1.4mm]}] (input.north) -- (x_ln.south);

\begin{scope}[local bounding box=c_ln]
  \fill[blue!15, rounded corners] (7.5,4.5) rectangle (12.5,5.5);
  \node at (10,5) {\Large LayerNorm/Modulation};
\end{scope}

\draw[->, thin, -{Stealth[length=2mm,width=1.4mm]}] (condition.north) -- (c_ln.south);
\draw[->, thin, -{Stealth[length=2mm,width=1.4mm]}] (c_ln.west) -- (x_ln.east);

\begin{scope}[local bounding box=mlp]
  \fill[blue!15, rounded corners] (13,1.5) rectangle (16,3.5);
  \node at (14.5,2.5) {\Large MLP};
\end{scope}

\draw[->, thin, -{Stealth[length=2mm,width=1.4mm]}] (time.north) -- (mlp.south);
\draw[->, thin, -{Stealth[length=2mm,width=1.4mm]}] (mlp.north) |- (c_ln.east);

\begin{scope}[local bounding box=c_pom]
  \fill[green!15, rounded corners] (1,6) rectangle (6,8);
  \node at (3.5,7) {\Large PoM (cross)};
\end{scope}

\draw[->, thin, -{Stealth[length=2mm,width=1.4mm]}] (x_ln.north) -- (c_pom.south);
\draw[->, thin, -{Stealth[length=2mm,width=1.4mm]}] (c_ln.north) |- (c_pom.east);

\begin{scope}[local bounding box=gate1]
  \draw[black, thick] (3.5,9) circle (0.3);
  \node at (3.5,9) {$\cdot$};
\end{scope}

\draw[->, thin, -{Stealth[length=2mm,width=1.4mm]}] (c_pom.north) -- (gate1.south);
\draw[->, thin, -{Stealth[length=2mm,width=1.4mm]}] (mlp.north) |- (gate1.east);

\begin{scope}[local bounding box=residual1]
  \draw[black, thick] (3.5,10) circle (0.3);
  \node at (3.5,10) {$+$};
\end{scope}

\draw[->, thin, -{Stealth[length=2mm,width=1.4mm]}] (gate1.north) -- (residual1.south);
\coordinate (midpoint) at ($(input.north)!0.5!(x_ln.south)$);
\draw[->, thin, -{Stealth[length=2mm,width=1.4mm]}] (midpoint) -- ++(-3,0) |- (residual1.west);

\begin{scope}[local bounding box=x_ln2]
  \fill[blue!15, rounded corners] (1,11.5) rectangle (6,12.5);
  \node at (3.5,12) {\Large LayerNorm/Modulation};
\end{scope}

\draw[->, thin, -{Stealth[length=2mm,width=1.4mm]}] (residual1.north) -- (x_ln2.south);
\draw[->, thin, -{Stealth[length=2mm,width=1.4mm]}] (mlp.north) |- (x_ln2.east);

\begin{scope}[local bounding box=s_pom]
  \fill[green!15, rounded corners] (1,13.5) rectangle (6,15.5);
  \node at (3.5,14.5) {\Large PoM (self)};
\end{scope}

\draw[->, thin, -{Stealth[length=2mm,width=1.4mm]}] (x_ln2.north) -- (s_pom.south);

\begin{scope}[local bounding box=gate2]
  \draw[black, thick] (3.5,16.5) circle (0.3);
  \node at (3.5,16.5) {$\cdot$};
\end{scope}

\draw[->, thin, -{Stealth[length=2mm,width=1.4mm]}] (s_pom.north) -- (gate2.south);
\draw[->, thin, -{Stealth[length=2mm,width=1.4mm]}] (mlp.north) |- (gate2.east);

\begin{scope}[local bounding box=residual2]
  \draw[black, thick] (3.5,17.5) circle (0.3);
  \node at (3.5,17.5) {$+$};
\end{scope}

\draw[->, thin, -{Stealth[length=2mm,width=1.4mm]}] (gate2.north) -- (residual2.south);

\coordinate (midpoint) at ($(residual1.north)!0.5!(x_ln2.south)$);
\draw[->, thin, -{Stealth[length=2mm,width=1.4mm]}] (midpoint) -- ++(-3,0) |- (residual2.west);

\draw[->, thin, -{Stealth[length=2mm,width=1.4mm]}] (residual2.north) -- ++(0,0.5);

\node[below=0.2cm of input] {\Large $X$};
\node[below=0.2cm of condition] {\Large $c$\color{white}{$]$}};
\node[below=0.2cm of time] {\Large $t$\color{white}{$]$}};

\end{tikzpicture}
        }
        \caption{Video block}
        \label{fig:video_block}
    \end{subfigure}
    \caption{\textbf{Building blocks for our diffusion models using PoM.} For class-conditional image generation (a), we follow strictly DiT\cite{peebles23iccv} in the AdaLN variant, replacing multi-head attention with PoM. For text to video generation (b), we follow a hybrid approach in which the encoded text tokens are incorporated into the video tokens using PoM instead of cross attention, while the time is used as a modulation. Modulation means component-wise scale and shift modification based on the coefficients predicted by the MLP (similarly to the AdaLN approach).}
    \label{fig:main}
\end{figure}
\begin{figure*}[tb]
    \centering
    \includegraphics[width=\textwidth]{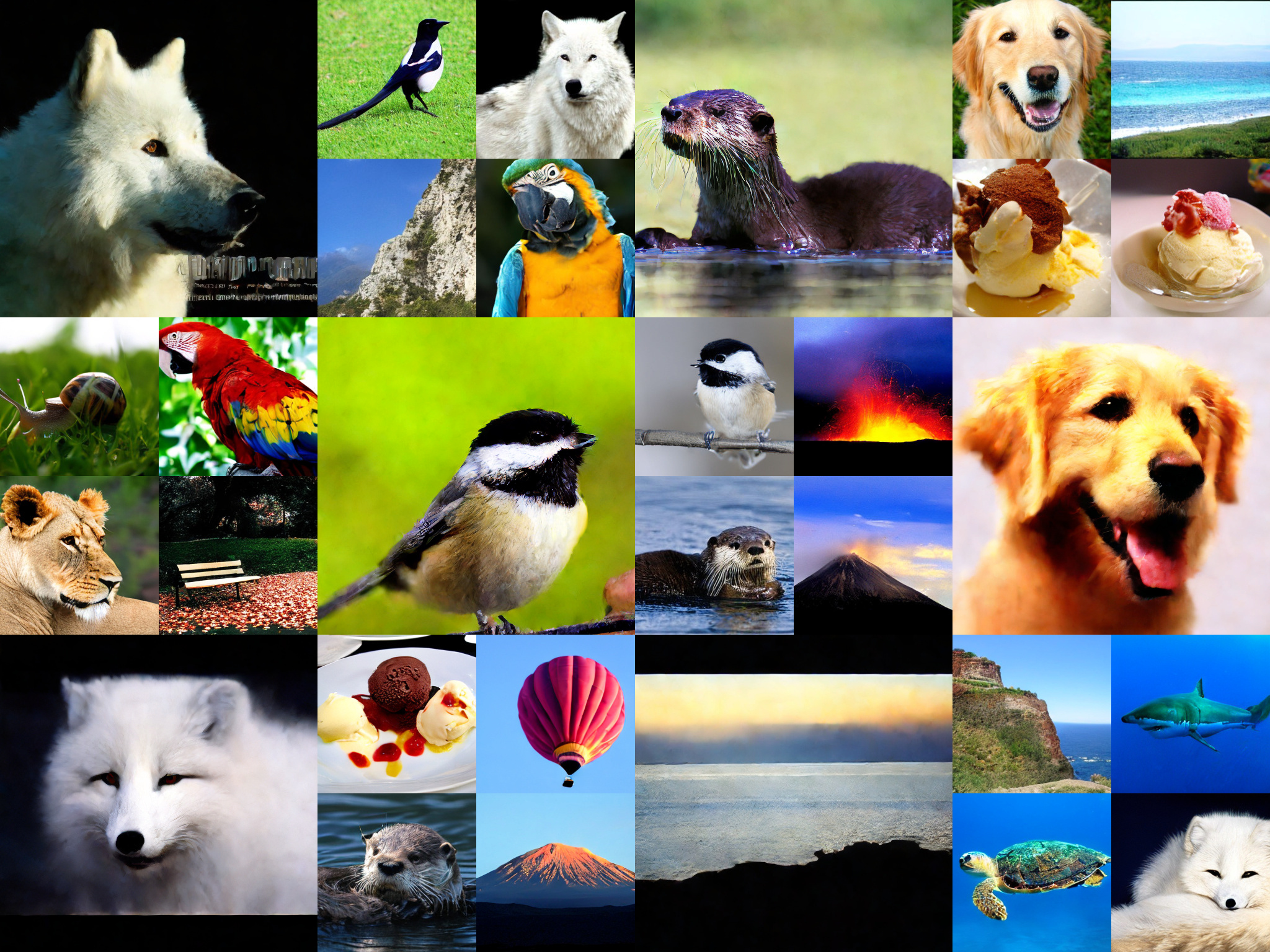}
    \caption{\textbf{Qualitative results on class-conditional generation}. We show images sampled with the model DiPoM-XL/2 trained with the flow-matching loss $\mathcal{L}_\text{FM}$ at several resolutions for different classes. We use classifier-free guidance with $\omega=4s/s_0$ with $s$ the scale of the image and $s_0$ the reference scale (256).}
    \label{fig:curated}
\end{figure*}
\section{Diffusion with PoM}
Armed with the definition of PoM and Polymorphers, we now design diffusion models taking inspiration from models based on MHA, and show that PoM can replace attention in practice.
We follow the design choices of DiT~\cite{peebles23iccv} and propose a class-conditional image generation polymorpher as well as a text-to-video generation polymorpher.

\subsection{Architecture design}
\paragraph{Image generation}
For image generation, the class-conditional polymorpher is similar to the AdaLN variant of DiT.
The image is encoded through the VAE of SD1.5~\cite{rombach22cvpr} and then features are aggregated into visual tokens $X$.
We add a 2D cosine positional encoding to them before we feed them to the model.
The class $c$ and the time step $t$ are embedded using an embedding matrix and a cosine embedding respectively before being summed together.

The model consists in several blocks that combine modulations, PoM and feed forward networks as shown on Figure \ref{fig:image_block}. In each block, the modulation consists in predicting from the condition $c+t$ a scale $\gamma$ and a shift $\beta$ that modify the input by
\begin{align}
    x \leftarrow \gamma(x - \beta).
\end{align}
Similarly to DiT, the MLP also predicts gates $\sigma$ that can shut down an entire block $f$ thanks to
\begin{align}
    x \leftarrow x + (1+s)f(x),
\end{align}
with the $1$ in $1+s$ being added so that there is a full residual connection when the MLP predicts $f(x)=0$.
For naming the architectures, we follow the same parametrization as in DiT. Namely, an \emph{S/2} model has a kernel size and stride of 2 for aggregating the VAE features into tokens, and 12 blocks of dimension 384.
Similarly, an XL/2 model that has 28 blocks of dimension 1152. For the PoM operation inside each block, we use an polynomial of order 2 with an expansion factor of 2 unless specified otherwise. Pytorch code for the blocks is given in appendix.

\paragraph{Video generation}
For video generation from text, we extend the DiT architecture to handle text as a condition. We first encode video clips using the 3D VAE from CogXVideo~\cite{yang2024cogvideox} and then group the features into visual tokens using a kernel size of $2\times 2 \times 2$ (with $2\times 2$ for the spatial axes, and $2$ for the temporal axis resulting in a downscaling factor of $16\times 16\times 8$).
We add a 3D cosine positional encoding to the visual tokens before feeding them to the model.
The text is encoded using T5~\cite{raffel2020exploring} embeddings and the time step is encoded using a cosine embedding.

The model consists in blocks using PoM to aggregate information between the text condition and the visual tokens as shown on Figure \ref{fig:video_block}. More precisely, a first PoM operation is used in a cross fashion, similar to cross-attention, to aggregate information from the text tokens into the visual tokens. Then, a second PoM operation is used to aggregate information among the visual tokens themselves, similar to what self-attention would do. Finally, a feed forward module processes the visual tokens only. The time step embedding is used in an MLP to predict the coefficients of modulations and gates at each of the operations.

We train a single model of size XL/2 that consists in 20 layers of dimension 1152 resulting in 1.1B parameters.

\subsection{Training setup}

For class-conditional image generation, we train on ImageNet. We rescale each image to 256 pixels on their smallest size and then take a crop of size $256\times 256$. We use both the original images and horizontally flipped version for a total of 2.4M images. We train a model $f_\theta$ either using the diffusion loss:
\begin{align}
    \mathcal{L}_\text{D} = E_{t\sim \mathcal{U}[0,1]} \| \varepsilon_t - f_\theta(x_t, c, t) \|^2,
\end{align}
or the flow matching loss:
\begin{align}
    \mathcal{L}_\text{FM} = E_{t\sim \mathcal{U}[0,1]} \| v_t - f_\theta(x_t, c, t) \|^2,
\end{align}
with $v_t = \varepsilon_t - x_0$. For each experimental result, we mention which loss is used, but the models are trained similarly without requiring change in training hyper-parameters. We use AdamW with a constant learning rate of $10^{-4}$ followed by a short cooldown with square root decay~\cite{hagele2024scaling}.

For video, we used WebVid-2M~\cite{Bain21} that we rescale to $240\times 384$ at 16 fps. We keep only the first 5 seconds, corresponding to 80 frames. 
This results in a total of 2.5M clips. We train using the flow matching loss $\mathcal{L}_\text{FM}$. We also use AdamW with a constant learning rate of $10^{-4}$ followed by a short cooldown with square root decay.

\begin{table*}[tb]
    \centering
    \begin{tabular}{lccccccc}
         \textbf{Model} & \textbf{Sample config} & \textbf{\#train} & \textbf{FID$\downarrow$} & \textbf{IS$\uparrow$} &\textbf{ Precision$\uparrow$} & \textbf{Recall$\uparrow$} \\
         Mask-GIT~\cite{chang2022maskgit} & &  & 6.18 & 182.1 & 0.80 & 0.51 \\ \hline
\rowcolor{gray!10}         DIFFUSSM-XL$^\dag$~\cite{yan24cvpr} & 250 steps DDPM &  660M & 2.28 & 259.1 & 0.86 & 0.56 \\
         DiM-H$^\dag$~\cite{teng2024dim} & 25 steps DPM++ &  480M & 2.21 & - & - & - \\ \hline
\rowcolor{gray!10}         ADM-G~\cite{dhariwal2021diffusion} & 250 steps DDIM   &  500M & 4.59 & 186.7 & 0.83 & 0.53 \\
         LDM-4-G~\cite{rombach22cvpr} & 250 steps DDIM  &  215M & 3.60 & 247.7 & 0.87 & 0.48 \\
\rowcolor{gray!10}         RIN~\cite{jabri23icml} & 1000 steps DDPM &  600M & 3.42 & 182.0 & - & - \\
         DiT-XL/2$^\dag$~\cite{peebles23iccv} &  250 steps DDPM &  1.8B & {\color{blue}\textbf{2.27}} & {\color{blue}\textbf{278.2}} & {\color{blue}\textbf{0.83}} & {\color{blue}0.57} \\
\rowcolor{gray!10}         SiT-XL/2$^\dag$~\cite{ma24eccv} & 125 steps Heun  &  1.8B & {\color{red}\textbf{2.15}} & {\color{red}254.9} & {\color{red}\textbf{0.81}} & {\color{red}\textbf{0.60}} \\ \hline
         DiPoM-XL/2 $\mathcal{L}_\text{D}$ (\textbf{ours}) & 250 steps DDIM &  950M & {\color{blue}2.46} & {\color{blue}240.6} & {\color{blue}0.78} & {\color{blue}\textbf{0.60}}\\
\rowcolor{gray!10}         DiPoM-XL/2 $\mathcal{L}_{\text{FM}}$ (\textbf{ours}) & 125 steps Heun &  950M & {\color{red}3.70} & {\color{red}\textbf{255.2}} & {\color{red}0.79} & {\color{red}0.56}
    \end{tabular}
    \caption{\textbf{Quantitative results on ImageNet $256\times 256$ class-conditional generation.} \#train denotes the number of training images seen during train (i.e., batch size $\times$ number of training steps). $\dag$ denotes methods evaluated against the Imagenet training set instead of the usual ADM evaluation archive. We color in {\color{blue}blue} (respectively in {\color{red}red}) DiT and PoM architectures of equivalent size that are trained with the same diffusion loss $\mathcal{L}_\text{D}$ (respectively flow-matching loss $\mathcal{L}_\text{FM}$), and we bold the best values between the two, even though the results are not evaluated against the same reference set.}
    \label{tab:img_sota}
\end{table*}
\begin{table}[tb]
    \begin{center}
    \begin{tabular}{cccccc}
        Degree & Expand & FID$\downarrow$ & IS$\uparrow$ & Precision & Recall\\
         \hline
\rowcolor{gray!10}        1 & 12 & 90.1  &  15.1 & 0.27 & 0.36 \\
        2 & 6  & 87.0  &  15.8 & \textbf{0.29} & 0.37\\
\rowcolor{gray!10}        3 & 4  & \textbf{86.1}  &  \textbf{16.0} & \textbf{0.29} & \textbf{0.38}\\
        4 & 3  & 88.8  &  15.5 & 0.28 & 0.36\\
\rowcolor{gray!10}        6 & 2  & 90.7  &  15.0 & 0.28 & 0.36 
    \end{tabular}
    \end{center}
    \caption{\textbf{Comparison of different degrees of Polynomial Mixer} at a constant memory budget with a S/2 model on 10k images from Imagenet. Having a degree $\geq 2$ is necessary to get good performances, but there is a trade-off between the degree and the expansion factor.}
    \label{tab:degree}
\end{table}
\begin{figure}[tb]
    \centering
    \begin{tikzpicture}
    \pgfplotsset{
        scale only axis,
        xmin=10, xmax=230,
        y axis style/.style={
            yticklabel style=#1,
            ylabel style=#1,
            y axis line style=#1,
            ytick style=#1
       }
    }
    \begin{axis}[
      width=0.7\columnwidth,
      xmode=log,
      axis y line*=right,
      axis x line=none,
      ymin=20, ymax=245,
      ylabel=IS,
      y axis style=orange!80,
      xmajorgrids=true,
      xminorgrids=true,
      major grid style={gray, dashed},
      minor grid style={gray!50, dotted}
    ]
    \addplot[only marks,mark=square,orange!80,mark size=3pt] 
      coordinates{
        (11.751,37.70355224609375)
        (46,75.90705108642578)  
        (145.758,159.68934631347656)
        (214.184, 214.12852478027344)
    };
    \addplot[orange!80,thick,style={line width=1.5pt}] table[
        y={create col/linear regression={y=y}}
    ] {
        x y
        11.751 37.70355224609375
        46 75.90705108642578
        145.758 159.68934631347656
        214.184 214.12852478027344
    };
    \end{axis}
    \begin{axis}[
      width=0.7\columnwidth,
      xmode=log,
      axis y line*=left,
      y axis style=blue!70,
      ymin=5, ymax=55,
      xlabel=Flops (G),
      ylabel=FID,
      xmajorgrids=true,
      xminorgrids=true,
      major grid style={gray, dashed},
      minor grid style={gray!50, dotted}
    ]
    \addplot[only marks,mark=o,blue!70,mark size=3pt] 
      coordinates{
        (11.751,46.42305612553264)
        (46,24.092694389677263) 
        (145.758, 12.542542810128566)
        (214.184, 8.775708317201747)
    };
    \addplot[blue!70,thick,style={line width=1.5pt}] table[
        y={create col/linear regression={y=y}}
    ] {
        x y
        11.751 46.42305612553264
        46 24.092694389677263
        145.758 12.542542810128566
        214.184 8.775708317201747
    };
    \end{axis}
    \end{tikzpicture}
    \caption{\textbf{Scaling laws for a DiT-like architecture with attention replaced by PoM}. FIDs and Inception Scores (IS) are computed on 10k samples with classifier free guidance ($\omega=1$), and shown with a linear regression in log space. Performances scale with the computation budget, similarly to transformers.}
    \label{fig:scaling_diffusion}
\end{figure}
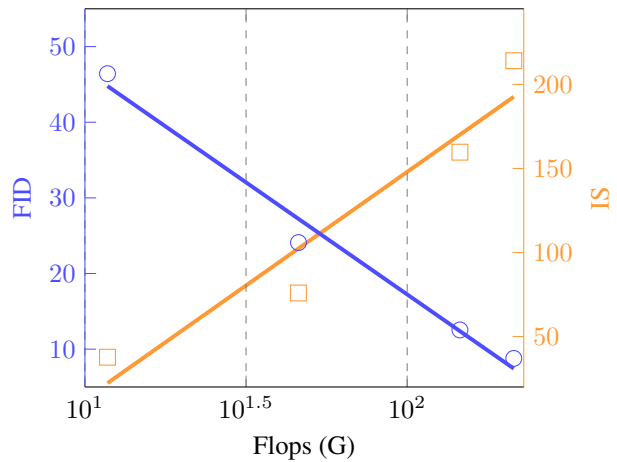

\begin{table*}[t]
\resizebox{\linewidth}{!}{
\begin{tabular}{lccccccccc}
\toprule
\multicolumn{1}{c}{\textbf{Models}} & \multicolumn{1}{c}{\textbf{Subject}} & \multicolumn{1}{c}{\textbf{Background}} & \multicolumn{1}{c}{\textbf{Temporal}} & \multicolumn{1}{c}{\textbf{Motion}} & \multicolumn{1}{c}{\textbf{Dynamic}} & \multicolumn{1}{c}{\textbf{Aesthetic}} & \multicolumn{1}{c}{\textbf{Imaging}} & \multicolumn{1}{c}{\textbf{Object}} \\
 & \multicolumn{1}{c}{\textbf{Consistency}} & \multicolumn{1}{c}{\textbf{Consistency}} & \multicolumn{1}{c}{\textbf{Flickering}} & \multicolumn{1}{c}{\textbf{Smoothness}} & \multicolumn{1}{c}{\textbf{Degree}} & \multicolumn{1}{c}{\textbf{Quality}} & \multicolumn{1}{c}{\textbf{Quality}} & \multicolumn{1}{c}{\textbf{Class}} \\
\midrule
\rowcolor{gray!10}\color{gray}  LaVie \cite{wang2023lavie} &\color{gray}  91.4\% &\color{gray} 97.5\% &\color{gray} 98.3\% &\color{gray} 96.4\% &\color{gray} 49.7\% &\color{gray} 54.9\% &\color{gray} 61.9\% &\color{gray} 91.8\% \\
 \color{gray} ModeScope \cite{wang2023modelscopetexttovideotechnicalreport} &\color{gray} 89.9\% &\color{gray} 95.3\% &\color{gray} 98.3\% &\color{gray} 95.8\% &\color{gray} 66.4\% &\color{gray} 52.1\% &\color{gray} 58.6\% &\color{gray} 82.3\% \\
\rowcolor{gray!10}  \color{gray} VideoCrafter \cite{he2022latent} &\color{gray} 86.2\% &\color{gray} 92.9\% &\color{gray} 97.6\% &\color{gray} 91.8\% &\color{gray} 89.7\% &\color{gray} 44.4\% &\color{gray} 57.2\% &\color{gray} 87.3\% \\
 \color{gray} CogVideo \cite{hong2023cogvideo} &\color{gray} 92.2\% &\color{gray} 96.2\% &\color{gray} 97.6\% &\color{gray} 96.5\% &\color{gray} 42.2\% &\color{gray} 38.2\% &\color{gray} 41.0\% &\color{gray} 73.4\% \\
 \rowcolor{gray!10} V-DiPoM-XL/2 \emph{no-mask} & 90.6\% & 96.6\% & 99.7\% & 97.3\% & 31.7\% & 28.6\% & 47.1\% & 29.3\% \\
 V-DiPoM-XL/2 \emph{b-causal} & 80.4\% & 92.2\% & 98.1\% &  97.4\% & 37.5\% & 30.5\% & 47.9\% & 30.0\% \\
\midrule
\multicolumn{1}{c}{\textbf{Models}} & \multicolumn{1}{c}{\textbf{Multiple}} & \multicolumn{1}{c}{\textbf{Human}} & \multicolumn{1}{c}{\textbf{Color}} & \multicolumn{1}{c}{\textbf{Spatial}} & \multicolumn{1}{c}{\textbf{Scene}} & \multicolumn{1}{c}{\textbf{Appearance}} & \multicolumn{1}{c}{\textbf{Temporal}} & \multicolumn{1}{c}{\textbf{Overall}} \\
 & \multicolumn{1}{c}{\textbf{Objects}} & \multicolumn{1}{c}{\textbf{Action}} & & \multicolumn{1}{c}{\textbf{Relationship}} & & \multicolumn{1}{c}{\textbf{Style}} & \multicolumn{1}{c}{\textbf{Style}} & \multicolumn{1}{c}{\textbf{Consistency}} \\
\midrule
\rowcolor{gray!10} \color{gray} LaVie \cite{wang2023lavie} &\color{gray} 33.3\% &\color{gray} 96.8\% &\color{gray} 86.4\% &\color{gray} 34.1\% &\color{gray} 52.7\% &\color{gray} 23.6\% &\color{gray} 25.9\% &\color{gray} 26.4\% \\
\color{gray} ModeScope \cite{wang2023modelscopetexttovideotechnicalreport} &\color{gray} 39/0\% &\color{gray} 92.4\% &\color{gray} 81.7\% &\color{gray} 33.7\% &\color{gray} 39.3\% &\color{gray} 23.4\% &\color{gray} 25.4\% &\color{gray} 25.8\% \\
\rowcolor{gray!10}\color{gray}  VideoCrafter \cite{he2022latent} &\color{gray} 25.9\% &\color{gray} 93.0\% &\color{gray} 78.8\% &\color{gray} 36.7\% &\color{gray} 43.4\% &\color{gray} 21.6\% &\color{gray} 25.4\% &\color{gray} 25.2\% \\
\color{gray} CogVideo \cite{hong2023cogvideo} &\color{gray} 18.1\% &\color{gray} 78.2\% &\color{gray} 79.6\% &\color{gray} 18.2\% &\color{gray} 28.2\% &\color{gray} 22.0\% &\color{gray} 7.8\% &\color{gray} 7.70\% \\
 \rowcolor{gray!10} V-DiPoM-XL/2 \emph{no-mask} & 1.9\% & 21.6\% & 76.3\% & 7.6\% & 2.8\% & 21.4\% & 13.4\% & 15.1\% \\
 V-DiPoM-XL/2 \emph{b-causal} & 3.3\% &  31.0\% & 69.5\% & 10.4\% & 3.0\% &  21.2\% & 17.2\% & 17.3\% \\
\end{tabular}
}
    \caption{\textbf{Quantitative results on VBench~\cite{huang24cvpr}}. We compare the same architecture of a V-DiPoM-XL/2 trained with the flow-matching loss $\mathcal{L}_\text{FM}$ using either no mask (denoted \emph{no-mask}) or block-causal masking (denoted \emph{b-causal}). We report results from the literature taken from~\cite{huang24cvpr} to provide some calibration, but noting that the comparison is not fair as these models are trained on much larger and richer datasets than ours, leading to much richer vocabulary and better semantic understanding.}
    \label{tab:vbench}
\end{table*}

\section{Experiments}
We first show result on class-conditional image generation and then of text-to-video generation.

\subsection{Class-conditional image generation}

\paragraph{Quantitative results}
We compare the results of our XL/2 model trained with the diffusion loss to the state of the art on Table~\ref{tab:img_sota}. We compute the Fréchet Inception Distance (FID), the Inception Score (IS), precision (P) and recall (R) using the code from ADM~\cite{dhariwal2021diffusion} on 50k generated images. The table is split between methods on masked encoding (Mask-GIT~\cite{chang2022maskgit}), diffusion models based on SSM and diffusion models based on attention. Results are extracted from the corresponding papers. Our images are generated with 250 steps of the DDIM sampler for the model trained with the diffusion loss $\mathcal{L}_\text{D}$, and 125 steps of Heun sampler for the model trained with the flow-matching loss $\mathcal{L}_\text{FM}$, with classifier free guidance (CFG, $\omega=0.7$ in both cases).

Using the evaluation code and reference set from ADM~\cite{dhariwal2021diffusion}, we obtain an FID of 2.46, which is slightly above that of the comparable DiT architecture, but notice that our model was trained for only half of the number steps of DiT. 
In addition, we found FID to be very unreliable as a metric, as it is highly varying with the reference set. 
For example, using the validation set of ImageNet, we obtained 3.45 FID\footnote{It was shown in ~\cite{picard21} that randomness affects significantly the results.}. 
We obtain a slightly lower IS compared to DiT, but this could be improved by using a higher CFG.
Indeed, we show in appendix that the trade-off between FID ans IS can reach as high as 300 IS at the cost of a much higher FID. 
We obtain a precision/recall trade-off comparable to DiT, slightly lower on precision but also higher on recall.

Overall, the results obtained using PoM are on par with the literature, showing that PoM can be used as a drop-in replacement for multi-head attention in a neural architecture, without requiring either architectural changes or training hyper-parameter tuning.

\paragraph{Qualitative results}
We further fine-tune the model on higher resolution data for a small number of steps to obtain a collection of models able to sample images up to $1024 \times 1024$ resolution (which is the maximum resolution we found reasonable to upscale to on ImageNet).
We show selected samples at these higher resolution on Figure~\ref{fig:curated}.
At higher resolution, some classes are collapsed due to the lack of available data.

\paragraph{Ablation study}
We study the impact of the degree of the polynomial on Table \ref{tab:degree}. 
To enable a fair comparison, we consider a set of S/2 models that have the same dimension for $H(X)$ and compute different trade-offs between the degree of the polynomials and their dimension.
As we can see, having at least second order polynomials is crucial to obtain the best performance.
This is consistent with the intuition that $H(X)$ has to contain sufficient statistics about the sequence and that using only the mean is not sufficient for that purpose.

We also study scaling laws for PoM by training models at different scales (S/2, B/2, L/2 and XL/2 following the DiT naming scheme), has shown on Figure~\ref{fig:scaling_diffusion}.
PoM enjoys exponential decrease of the FID with respect to the sampling computing complexity as shown by a linear regression on the logarithmic plot.
This is similar to what was observed for transformers in DiT~\cite{peebles23iccv}.

\subsection{Text to video generation}

We evaluate our model generating videos of 5 seconds at 16 fps and 240p resolution on VBench~\cite{huang24cvpr} and show the results in Table~\ref{tab:vbench}.
Note that contrarily to ImageNet, video generation is not as well standardized and models differ dramatically in terms of size, complexity and training dataset.
Notably, most text-to-video generation models are trained on a mix of images and videos to get more diverse captions.
In our case, we want to study the impact of enforcing temporal causality in the generation process and as such we limit our train set to WebVid-2M~\cite{Bain21} only.
Due to this smaller training set, we observed that our models are limited to a smaller vocabulary of objects, motion and styles.

We compare the standard architecture (denoted \emph{no-mask}) with the use of a block-causal mask as detailed in section~\ref{sec:causal} (denoted as \emph{b-causal}).
As we can see, the impact of using a block causal mask is negative on some tasks like \emph{subject consistency}, \emph{background consistency} and \emph{color}.
This can be explain by the model struggling to follow the prompt for the first frames in the \emph{block causal} case, which penalizes consistency, whereas the \emph{no-mask} case can leverage information from later frames to improve consistency.
Interestingly, using a \emph{block causal} mask improves temporal tasks like \emph{dynamic degree}, \emph{human action} and \emph{temporal style}, which shows the importance of modeling properly the temporal aspect for these tasks.

\section{Discussion}
In this paper, we presented PoM, the Polynomial Mixer, a building block for neural networks that aims at replacing attention.
PoM has a complexity linear with the sequence length and we prove it is a universal senquence-to-sequence approximator.
To demonstrate the effectiveness of PoM, we train image and video generation models with it in lieu of Multi-Head Attention. These diffusion models obtain competitive results while being able to generate higher resolution images faster than with attention.

PoM is very interesting for high-definition video of long duration.
However, the extreme cost of training such model makes this endeavor clearly out of the scope of a research paper.
Another area where PoM could shine is LLMs and more particularly multimodal LLMs.
Indeed, LLMs are causal, which means the generation of text could greatly benefit from the $\mathcal{O}(1)$ complexity of PoM for causal sequence.
In addition, recent works~\cite{zhou2024transfusionpredicttokendiffuse} show that next token prediction and diffusion objectives can be merged in a single model. In that case the ability of PoM to seamlessly adapt from causal to block-causal masking scheme greatly reduces the complexity of such mixed training objective.
As for high definition video, the extreme cost of training such large models also renders this endeavor out of the scope of a research paper.

\section{Acknowledgment}
This work was granted access to the HPC resources of IDRIS under the allocation 2024-AD011013085R2 made by GENCI. The authors would like to thank Vincent Lepetit, Gül Varol, Loic Landrieu and Dimitris Samaras for their insightful comments and suggestions.

{
    \small
    \bibliographystyle{splncs04}
    \bibliography{main}
}

\clearpage
\appendix
\section{PoM pytorch code}
In this section, we provide code in Pytorch for the main parts of the Polynomial Mixer as well as our diffusion blocks.

We found that writing dedicated functions for specific degrees led to faster runtime due to the ability of the PyTorch's compiler to optimize them.
We show below implementation for degrees 2, 3 and 4.
\begin{lstlisting}[language=Python, caption=Pytorch code for order specific PoM functions.]
@torch.compile
def po2(x: torch.Tensor):
    h1, h2 = gelu(x).chunk(2, dim=-1)
    h2 = h2 * h1
    return torch.cat([h1, h2], dim=-1)

@torch.compile
def po3(x: torch.Tensor):
    h1, h2, h3 = gelu(x).chunk(3, dim=-1)
    h2 = h2 * h1
    h3 = h3 * h2
    return torch.cat([h1, h2, h3], dim=-1)

@torch.compile
def po4(x: torch.Tensor):
    h1, h2, h3, h4 = gelu(x).chunk(4, dim=-1)
    h2 = h2 * h1
    h3 = h3 * h2
    h4 = h4 * h3
    return torch.cat([h1, h2, h3, h4], dim=-1)
\end{lstlisting}

Next, we show the function that computes both the polynomial and the mixing depending on the degree and the presence of a mask.
\begin{lstlisting}[language=Python, caption=Pytorch code for the complete polynomial and mixing part.]
def high_order_aggregation_(x: torch.Tensor, k: int, mask=None):
    if k == 2:
        h = po2(x)
    elif k == 3:
        h = po3(x)
    elif k == 4:
        h = po4(x)
    else:
        h = list(gelu(x).chunk(k, dim=-1))
        for i in range(1, k):
            h[i] = h[i] * h[i-1]
        h = torch.cat(h, dim=-1)
    if mask is None:
        h = h.mean(dim=1, keepdims=True)
    else:
        if mask.dim()==2:
            h = mask_mixer(h, mask.to(h.device))
        elif mask.dim() ==3:
            h = full_mask_mixer(h, mask.to(h.device))
        else:
            raise Exception('unsupported dim for mask (should be 2,3 or None)')
    return h
\end{lstlisting}

In the case the mask is 3 dimensional (batch, queries, context), we have a dedicated function that performs the partial sums. Note that this implementation is not optimized and that more speedup could be gained with a compiled mask.
\begin{lstlisting}[language=Python, caption=Pytorch code for the mixer part with full mask.]
def full_mask_mixer(h, mask):
    mask = mask.type(h.dtype)
    h = torch.einsum('bnd, bmn -> bmd', h, mask)  # b batch, n context tokens, m query tokens, d dim
    h = h / (1.e-7 + mask.sum(dim=2, keepdims=True))
    return h
\end{lstlisting}

The selection operation is very simple and consists in an element-wise product. The whole PoM operation is just the computation of $H(X)$ followed by the selection.
\begin{lstlisting}[language=Python, caption=Pytorch code for the selection part and the whole PoM function.]
@torch.compile
def high_order_selection_(x: torch.Tensor, h: torch.Tensor):
    return F.sigmoid(x) * h

def pom(xq: torch.Tensor, xc: torch.Tensor, k: int, mask=None):
    h = high_order_aggregation_(xc, k, mask)
    o = high_order_selection_(xq, h)
    return o
\end{lstlisting}

In the PoM module, we add the projections $W_{1\dots m}, W_s$ and $W_o$ for each part of the PoM operation.
\begin{lstlisting}[language=Python, caption=Pytorch module for PoM.]
class PoM(nn.Module):
    def __init__(self, dim, order, order_expand, bias=True):
        super().__init__()
        self.dim = dim
        self.order = order
        self.order_expand = order_expand
        self.ho_proj = nn.Linear(dim, order*order_expand*dim, bias=bias)
        self.se_proj = nn.Linear(dim, order*order_expand*dim, bias=bias)
        self.ag_proj = nn.Linear(order*order_expand*dim, dim, bias=bias)
        self.hom = hom

    def forward(self, xq, xc=None, mask=None):
        if xc is None:
            xc = xq # self attention

        s = self.se_proj(xq)
        h = self.ho_proj(xc)
        sh = self.hom(s, h, self.order, mask)

        # output projection
        return self.ag_proj(sh)
\end{lstlisting}

For image diffusion, the base building block is simply a PoM module followed by an MLP, with residual connections and AdaLN modulations.
\begin{lstlisting}[language=Python, caption=Pytorch module for the image diffusion block.]
def modulation(x, scale, bias):
    return x * (1+scale) + bias
    
class DiPBlock(nn.Module):
    def __init__(self, dim: int, order: int, order_expand: int, ffw_expand: int):
        super().__init__()
        self.dim = dim
        self.order = order
        self.order_expand = order_expand
        self.ffw_expand = ffw_expand

        self.mha_ln = nn.LayerNorm(dim, elementwise_affine=False, eps=1e-6)
        self.pom = PoM(dim, order=order, order_expand=order_expand, bias=True)
        self.ffw_ln = nn.LayerNorm(dim, elementwise_affine=False, eps=1e-6)
        self.ffw = nn.Sequential(nn.Linear(dim, ffw_expand * dim, bias=True),
                                 nn.GELU(),
                                 nn.Linear(ffw_expand * dim, dim, bias=True))
        self.cond_mlp = nn.Sequential(
                                 nn.SiLU(),
                                 nn.Linear(dim, 4 * dim, bias=True))
        self.gate_mlp = nn.Sequential(
                                 nn.SiLU(),
                                 nn.Linear(dim, 2 * dim, bias=True))


    def forward(self, x, c):
        s1, b1, s2, b2 = self.cond_mlp(c).chunk(4, -1)
        g1, g2 = self.gate_mlp(c).chunk(2, -1)

        # mha
        x_ln = modulation(self.mha_ln(x), s1, b1)
        x = x + self.pom(x_ln) * (1 + g1)

        #ffw
        x_ln = modulation(self.ffw_ln(x), s2, b2)
        x = x + self.ffw(x_ln)*(1+g2)

        return x
\end{lstlisting}

For text-to-video, we add a second PoM module that gathers information from the text.
\begin{lstlisting}[language=Python, caption=Pytorch module for the video diffusion block.]
class TextVideoDiPBlock(nn.Module):
    def __init__(self, dim: int, order: int, order_expand: int, ffw_expand: int):
        super().__init__()
        self.dim = dim
        self.order = order
        self.order_expand = order_expand
        self.ffw_expand = ffw_expand

        self.mha_ln = nn.LayerNorm(dim, elementwise_affine=False, eps=1e-6)
        self.x_mha_ln = nn.LayerNorm(dim, elementwise_affine=False, eps=1e-6)
        self.c_mha_ln = nn.LayerNorm(dim, elementwise_affine=False, eps=1e-6)
        self.pom = PoM(dim, order=order, order_expand=order_expand, bias=True)
        self.c_pom = PoM(dim, order=order, order_expand=order_expand, bias=True)
        self.ffw_ln = nn.LayerNorm(dim, elementwise_affine=False, eps=1e-6)
        self.ffw = nn.Sequential(nn.Linear(dim, ffw_expand * dim, bias=True),
                                 nn.GELU(),
                                 nn.Linear(ffw_expand * dim, dim, bias=True))
        self.cond_mlp = nn.Sequential(
                                 nn.SiLU(),
                                 nn.Linear(dim, 8 * dim, bias=True))
        self.gate_mlp = nn.Sequential(
                                 nn.SiLU(),
                                 nn.Linear(dim, 3 * dim, bias=True))


    def forward(self, x, t, c, mask, temporal_mask=None):
        sx, bx, sc, bc, s1, b1, s2, b2 = self.cond_mlp(t).chunk(8, -1)
        gc, g1, g2 = self.gate_mlp(t).chunk(3, -1)

        # ca
        x_ln = modulation(self.x_mha_ln(x), sx, bx)
        c_ln = modulation(self.c_mha_ln(c), sc, bc)
        x = x + self.c_pom(x_ln, c_ln, mask) * (1 + gc)

        # sa
        x_ln = modulation(self.mha_ln(x), s1, b1)
        x = x + self.pom(x_ln, mask=temporal_mask) * (1 + g1)

        #ffw
        x_ln = modulation(self.ffw_ln(x), s2, b2)
        x = x + self.ffw(x_ln)*(1+g2)

        return x
\end{lstlisting}

\section{Condition adherence}
In this section we study the trade-off between image quality as measured with FID and condition adherence as measured with Inception Score (IS) by varying the weight $\omega$ of the classifier-free guidance (CFG).
We show the results for a model of size L2 trained with the diffusion loss $\mathcal{L}_\text{D}$ on Gigure~\ref{fig:fidis}.
Inference is performed with 250 steps of DDIM sampling.
As we can see, the model is perfectly able to balance FID and IS, leading to a typical 'U' curve where CFG improves both FID and IS at first, but then improvements of IS comes at the cost of FID.
This is typical of mode collapse with the model generating low diversity but high quality images, similarly to what is observed with attention-based models.

\section{Proof of Lemma 3}
    We first need to show that set with different entries are mapped to different vectors. We first separate $\PoM$ into its two components:
    \begin{align}
        s(X) = \sigma(W_s X)\\
        H_k(X) = \left[ h(W_1 X); \dots; \prod_m h(W_m X)\right] \bm{11}^\top\\
        \PoM(X) = W_o (s(X) \circ H_k(X))
    \end{align}

    Assuming $\ker(W_o) = \empty$, and noting that $H_k(X)$ is the same for every column, we just have to show that $s(X)$ has different columns. This is easily achieved by having $\ker(W_s) = \empty$ since $\sigma$ is injective and the composition of injective functions is itself injective.

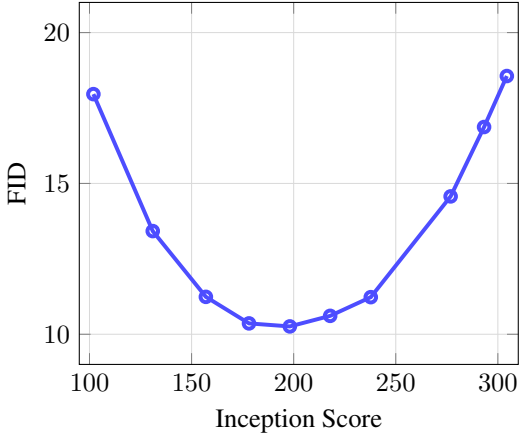
\begin{figure}
    \centering
    \begin{tikzpicture}
\begin{axis}[
    width=0.89\linewidth,
    xlabel=Inception Score,
    ylabel=FID,
    grid=both,
    grid style={line width=.1pt, draw=gray!10},
    major grid style={line width=.2pt, draw=gray!30},
    xmin=95, xmax=310,
    ymin=9, ymax=21
            ]
\addplot[solid,mark=o,blue!70,style={line width=1.5pt}] plot coordinates {
    (101.91,17.96)
    (130.98,13.42)
    (157.00,11.24)
    (178.12,10.36)
    (198.12,10.26)
    (217.87,10.61)
    (237.76,11.23)
    (276.87,14.57)
    (293.25,16.87)
    (304.40,18.56)
};
\end{axis}
    \end{tikzpicture}
    \caption{\textbf{Image quality versus condition adherence trade-off}. FID/IS curve for the L2 model with 250 DDIM sampling steps. Values are computed on 10k images against the validation set of ImageNet.}
    \label{fig:fidis}
\end{figure}

    Second, we have to show that sets that differ by at least one element are mapped to all different entries. To simplify notations, we will consider the special case where all matrices are the identity or an identity block positioned such as to perform submatrix selection. All the matrices can thus be removed from the formula. A similar argument can be made for matrices that are full rank as they preserve injectivity. We will also consider linear activations everywhere, which can be made as close as one wish by partitioning the image of the activation function and performing piecewise linear approximation.

    With this simplified version of $\PoM$, we have to show that for 2 sets $X,X'$ differing by at least one element (i.e., $\exists x' \in X', \forall x\in X, x \neq x'$), then there exist $k$ such that
    \begin{align}
        \forall x\in X, x'\in X', x\sum_{x_i\in X} x_i^k \neq x\sum_{x_i\in X} x_i^k.
    \end{align}

    Consider the functions $P(t)$ and $P'(t)$ defined as follows:
    \begin{align}
        P(t)=\sum_{x_i in X} x_i^t\\
        P'(t)=\sum_{x_i in X'} x_i^t
    \end{align}

    Since $X$ and $X'$ differ by at least one element, there exists at least one $x_i \in X$ such that $x_i \neq x_i', \forall x_i' \in X'$. This implies that the functions $P(t)$ and $P'(t)$ are not identical since are sums of exponentials with different bases.

    Since $P(t)$ and $P'(t)$ are different functions, there must exist some $k$ for which $P(k) \neq P'(k)$. In other words, there exists a $k$ such that:
    \begin{align}
        \sum_{x_i\in X} x_i^k \neq \sum_{x_i' \in X'} x_i'^k
    \end{align}
    
    For this $k$, let us denote $S_k = \sum_{x_i\in X}x_i^k$. We need to show that $xS_k \neq x'S_k'$ for all $x\in X$ and $x'\in X'$.
    Assume for the sake of contradiction that there exist $x\in X$ and $x'\in X'$ such that $xS_k = x'S_k'$. This implies:
    \begin{align}
        x \sum_{x_i \in X} x_i^k = x' \sum_{x_i' \in X'} x_i'^k
    \end{align}

    Rearranging, we get:
    \begin{align}
        \frac{x}{x'} =  \frac{\sum_{x_i' \in X'}x_i'^k}{\sum_{x_i\in X}x_i^k}
    \end{align}

    Since $S_k\neq S_k'$, the right-hand side is not equal to 1. However, for this equality to hold for all $x\in X$ and $x'\in X'$, the ratio $x/x'$ would need to be constant for all pairs $(x,x')$, which is not possible given that $X$ and $X'$ differ by at least one element.

    Therefore, there exists a $k$ such that $x S_k\neq x'S_k'$ for all $x\in X$ and $x'\in X'$.

\section{Uncurated examples}
In the following pages, we show randomly selected samples with obtained after 250 steps of DDIM sampling with the XL/2 model trained with the diffusion loss $\mathcal{L}_\text{D}$.

\newcommand{\addimages}[1]{%
\readarraysepchar{,}
\readdef{#1/images.txt}\imageList
\readarray*\imageList\imageArray[-,1]
\begin{tikzpicture}
    \pgfmathsetmacro{\scaleFactor}{\textwidth / (5 * 256)}

    \foreach \i in {1,...,30} {
        \pgfmathsetmacro{\row}{int(floor((\i - 1) / 5))} 
        \pgfmathsetmacro{\col}{int(mod(\i - 1, 5))} 

        \edef\image{\imageArray[\i,1]}
        \pgfmathsetmacro{\xcoord}{\col * 256 * \scaleFactor}
        \pgfmathsetmacro{\ycoord}{-\row * 256 * \scaleFactor}

        \typeout{Coordinates: (\xcoord, \ycoord)}
        \pgfmathsetmacro{\imgsize}{256 * \scaleFactor}

        \node[inner sep=0pt] at (\xcoord pt, \ycoord pt) {
            \includegraphics[width=\imgsize pt, height=\imgsize pt]{#1/\image}
        };
    }
\end{tikzpicture}
}

\begin{figure*}
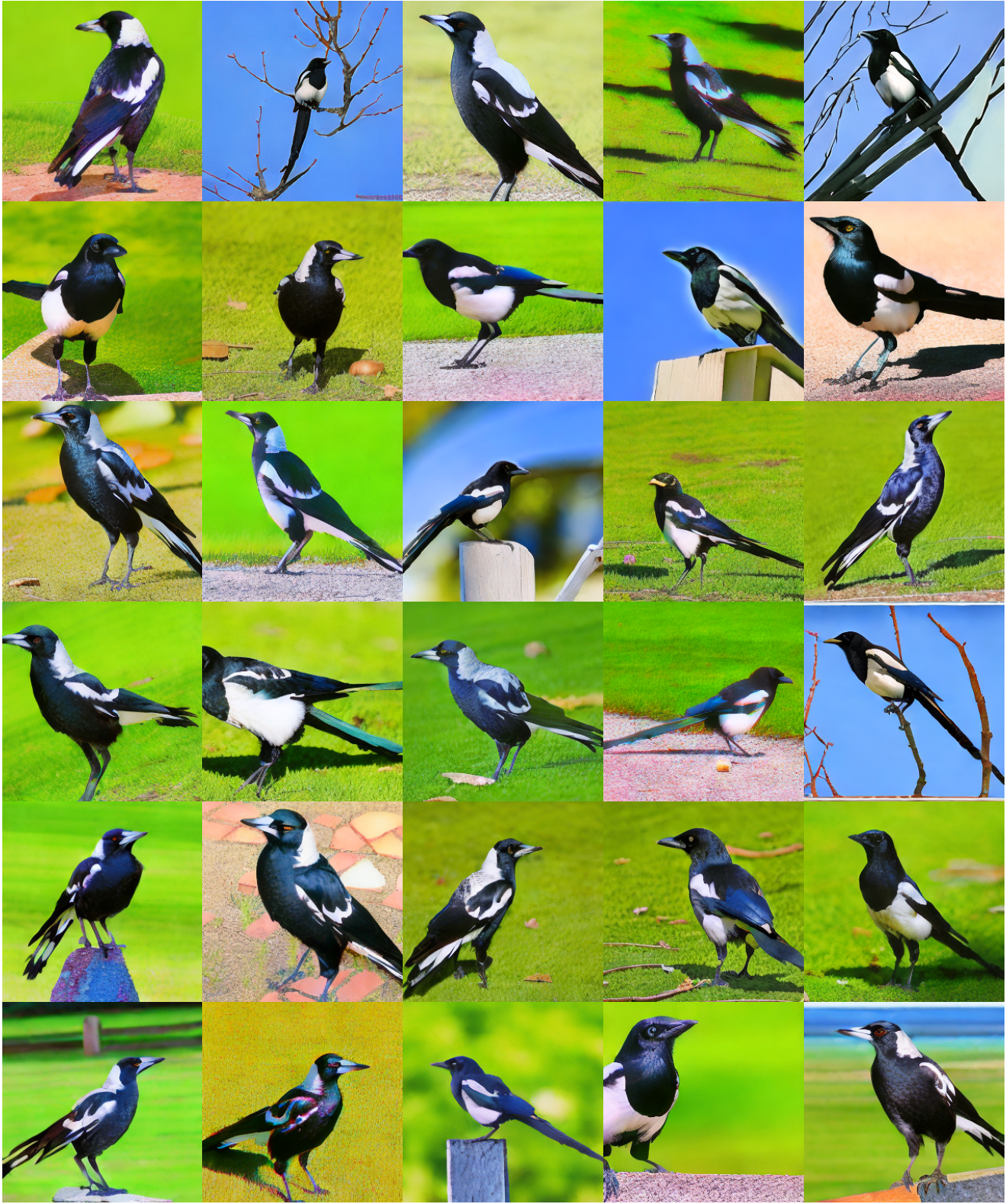

    \begin{center}
    \addimages{images/18/}
    \end{center}
    \caption{Uncurated 256² images for the class \emph{magpie} (18).}
    \label{fig:magpie}
\end{figure*}

\begin{figure*}
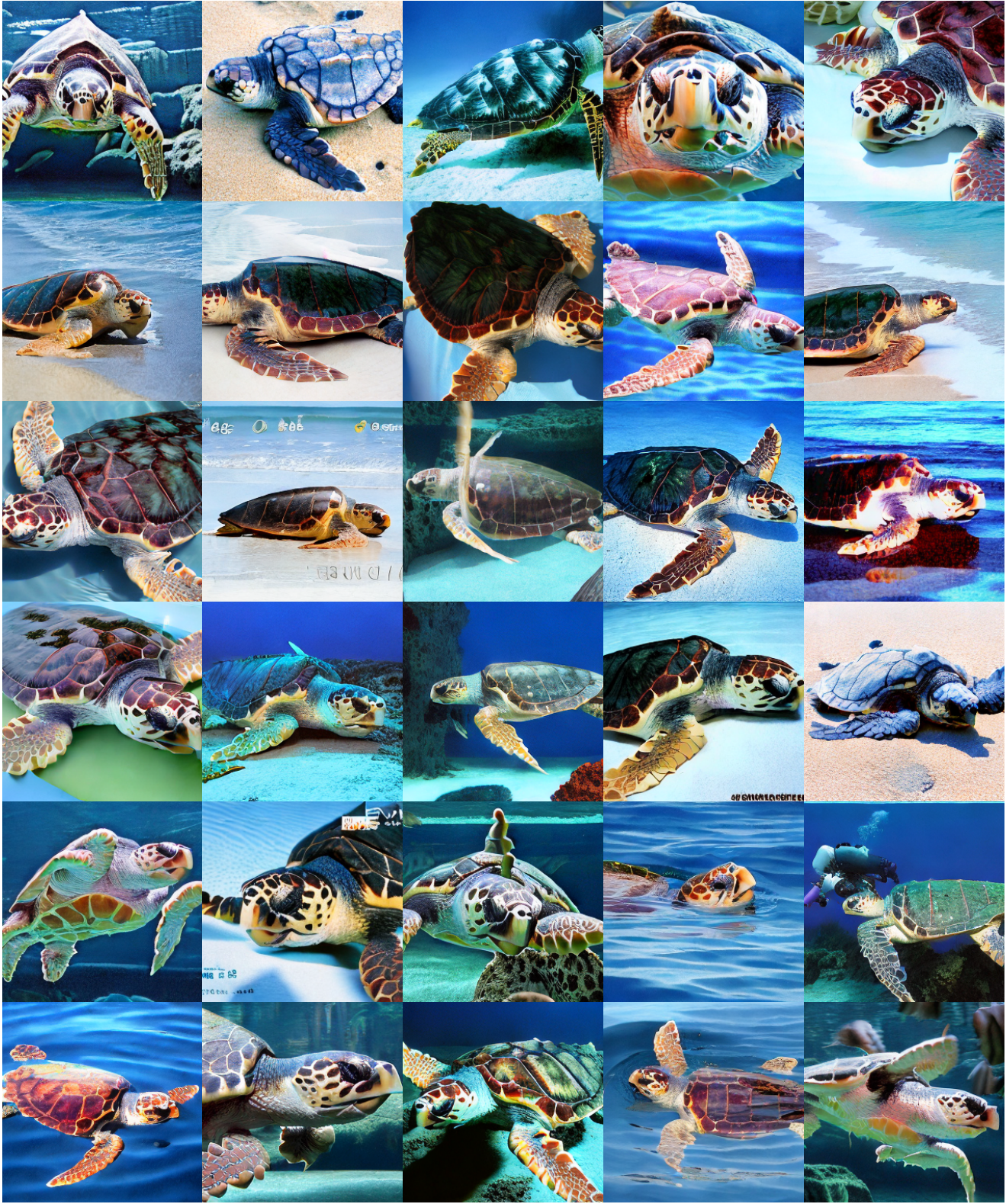

    \begin{center}
    \addimages{images/33/}
    \end{center}
    \caption{Uncurated 256² images for the class \emph{loggerhead, loggerhead turtle, Caretta caretta} (33).}
    \label{fig:turtle}
\end{figure*}

\begin{figure*}
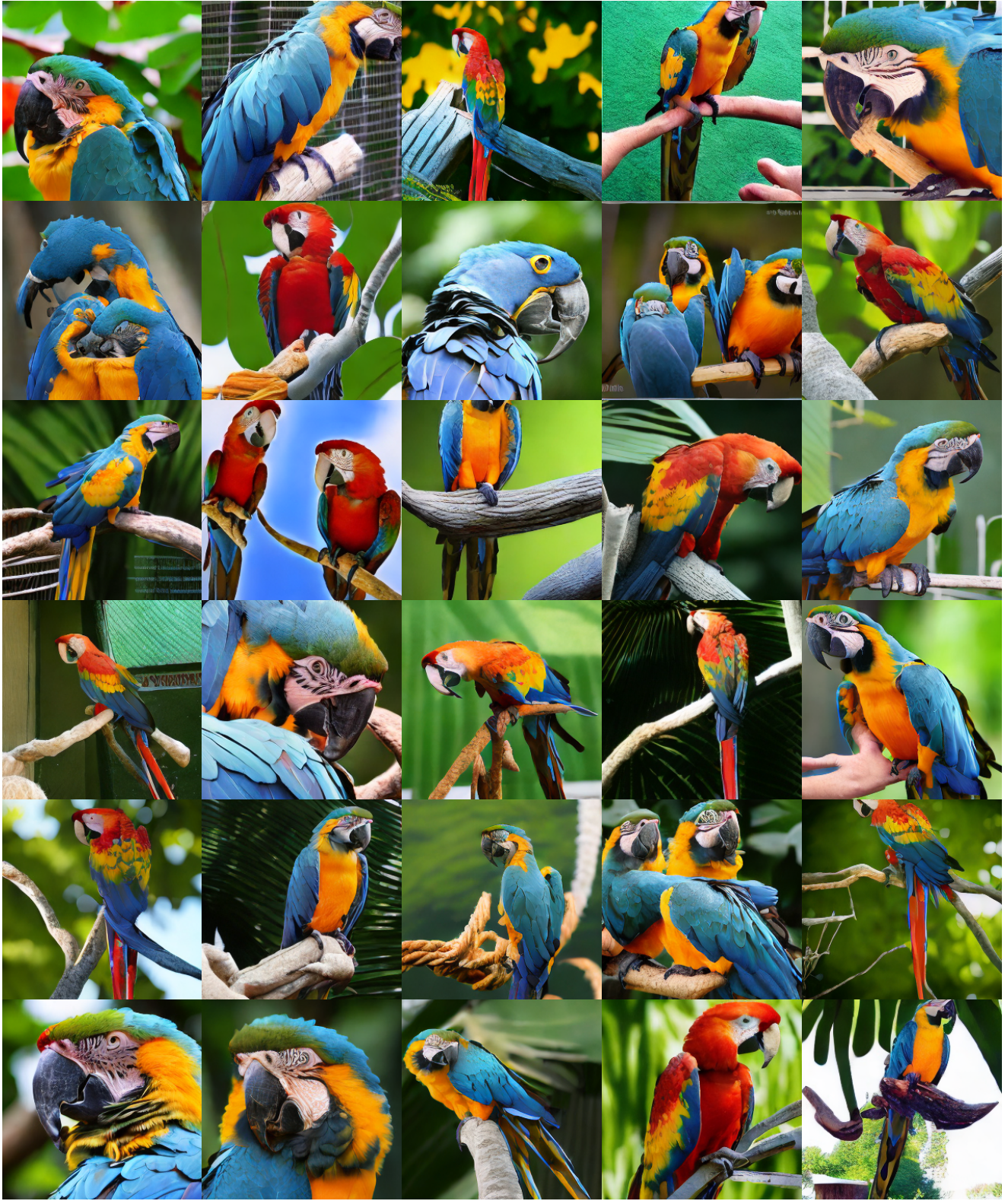

    \begin{center}
    \addimages{images/88/}
    \end{center}
    \caption{Uncurated 256² images for the class \emph{macaw} (88).}
    \label{fig:macaw}
\end{figure*}

\begin{figure*}
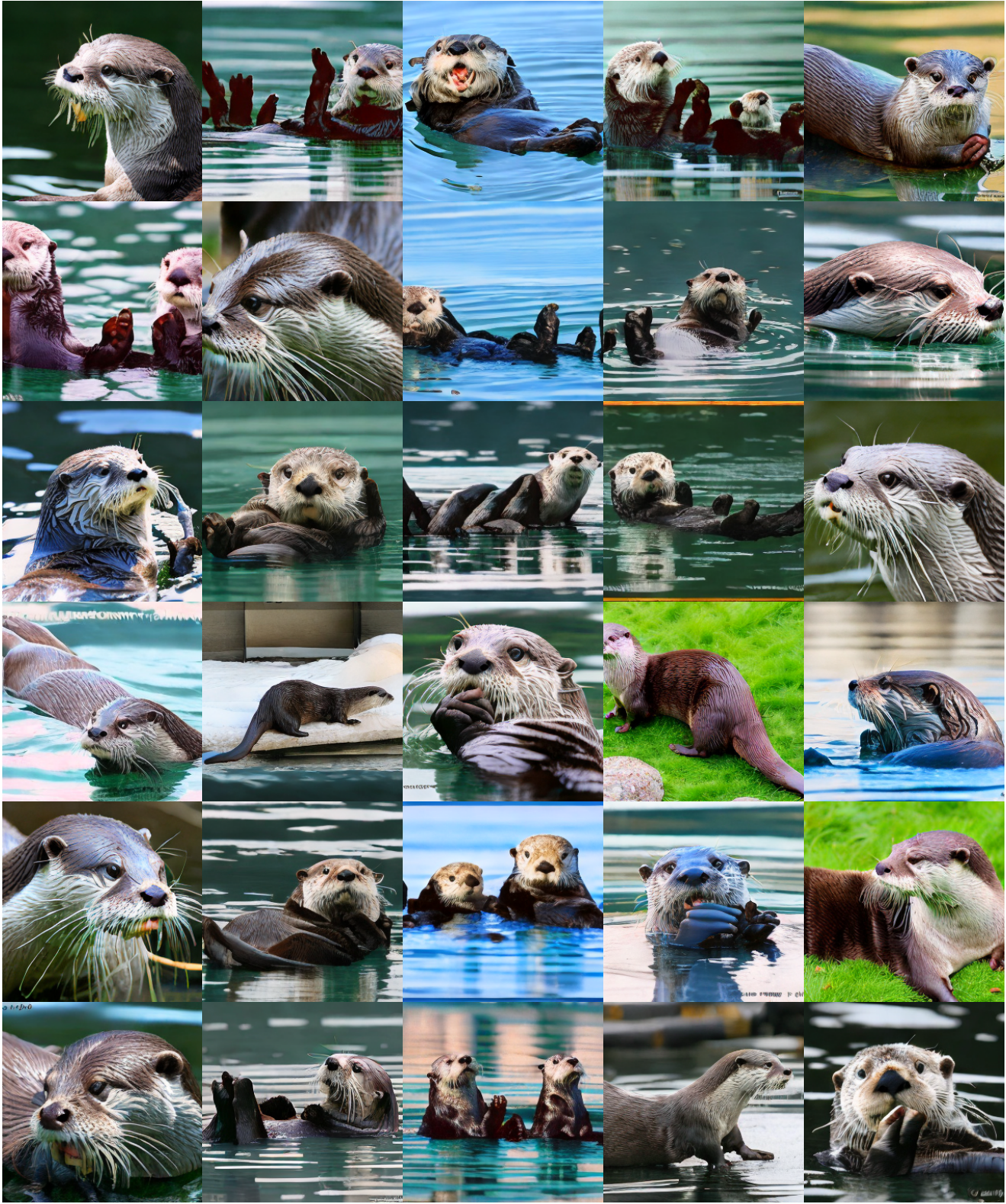

    \begin{center}
    \addimages{images/360/}
    \end{center}
    \caption{Uncurated 256² images for the class \emph{otter} (360).}
    \label{fig:otter}
\end{figure*}

\begin{figure*}
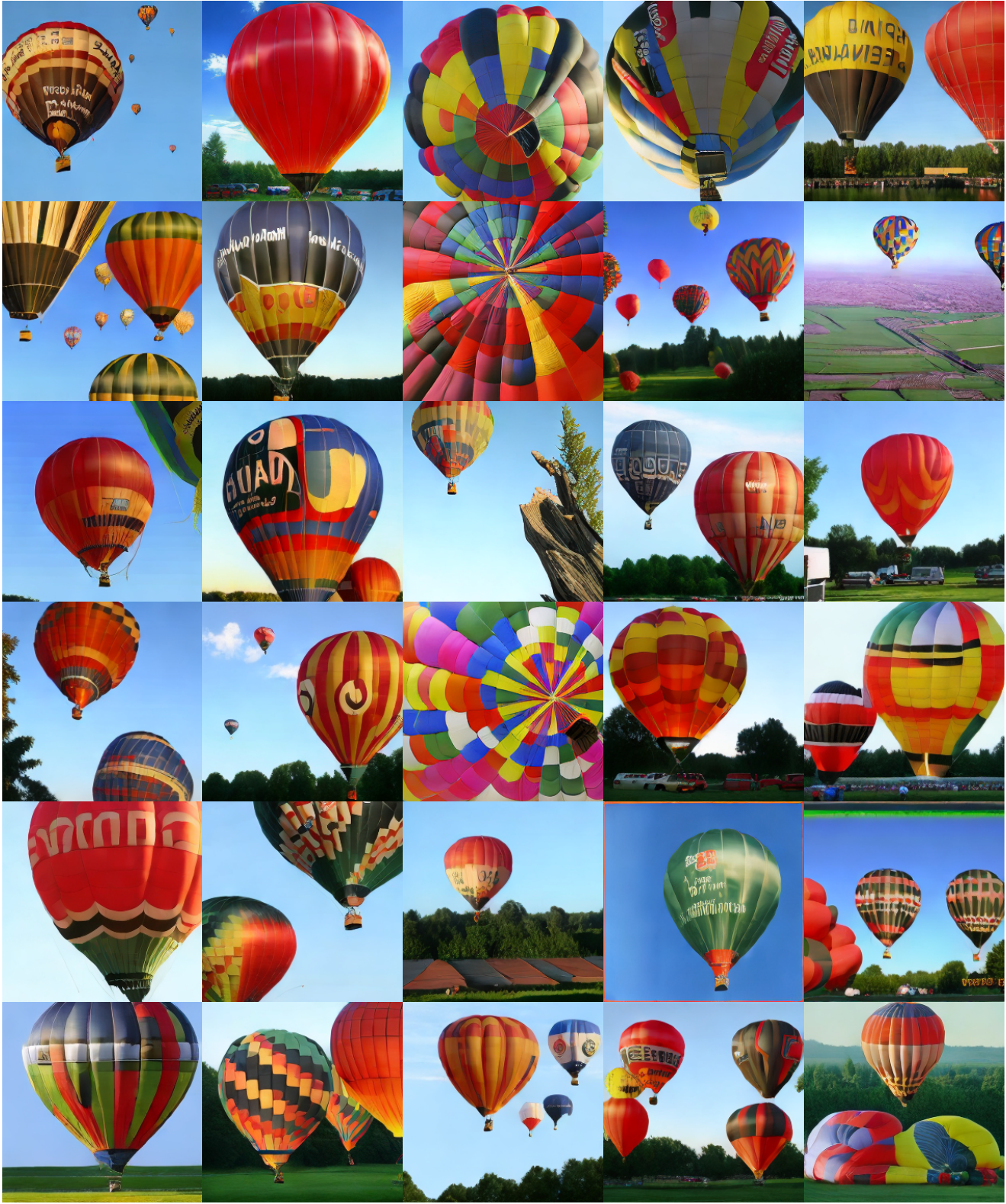

    \begin{center}
    \addimages{images/417/}
    \end{center}
    \caption{Uncurated 256² images for the class \emph{balloon} (417).}
    \label{fig:balloon}
\end{figure*}

\begin{figure*}
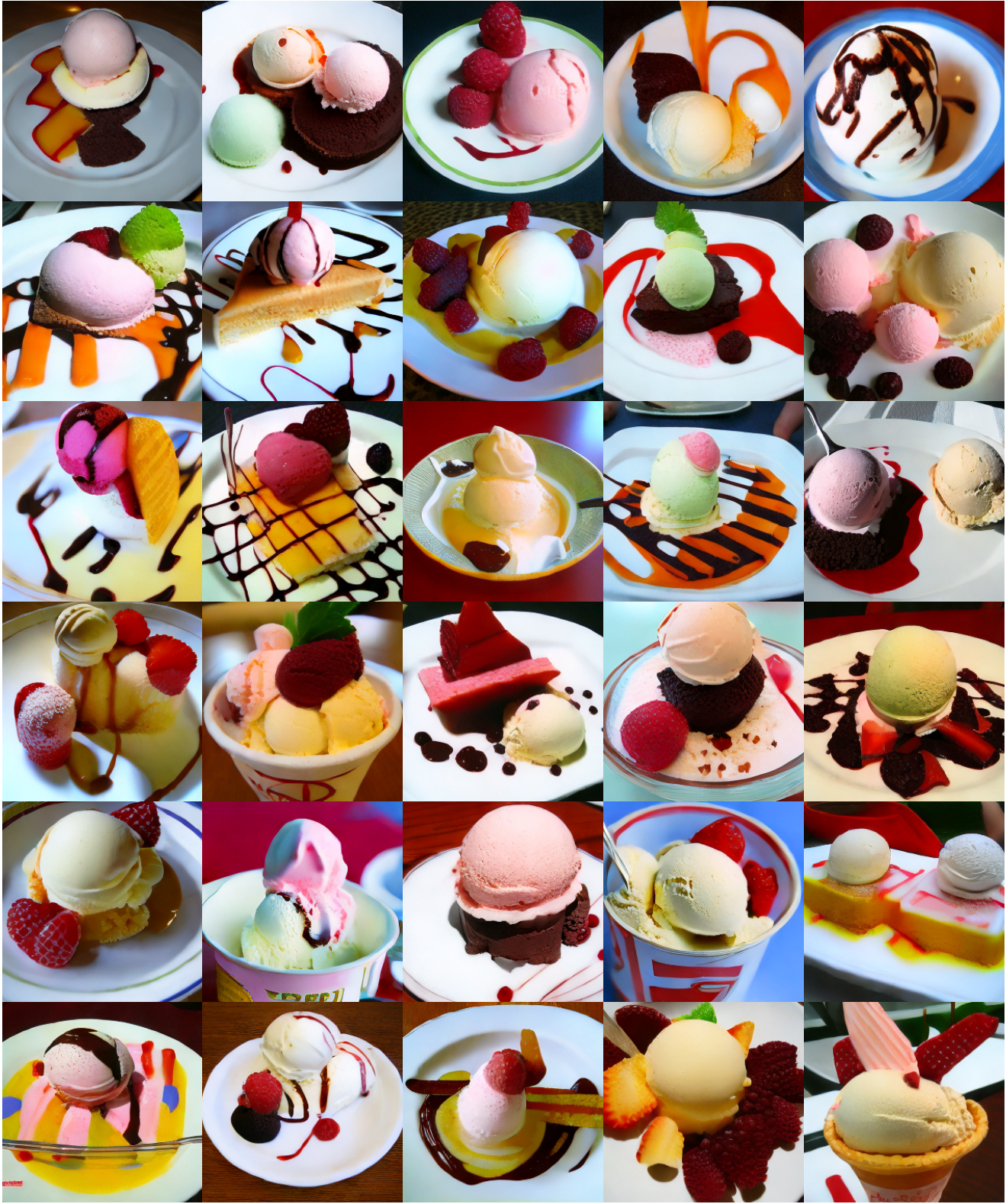

    \begin{center}
    \addimages{images/928/}
    \end{center}
    \caption{Uncurated 256² images for the class \emph{ice cream, icecream} (928).}
    \label{fig:icecream}
\end{figure*}

\begin{figure*}
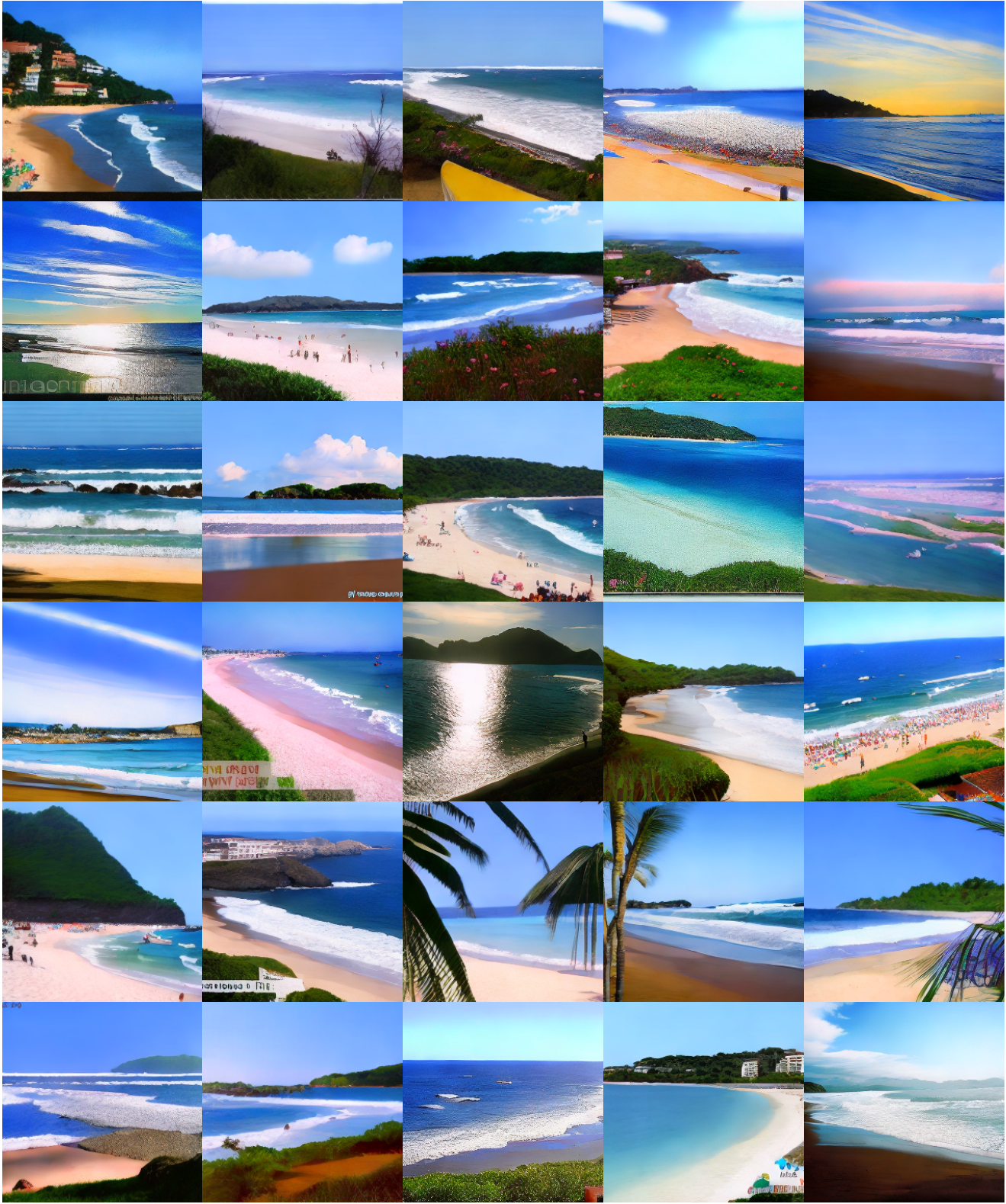

    \begin{center}
    \addimages{images/978/}
    \end{center}
    \caption{Uncurated 256² images for the class \emph{seashore, coast, seacoast, sea-coast} (978).}
    \label{fig:seashore}
\end{figure*}

\begin{figure*}
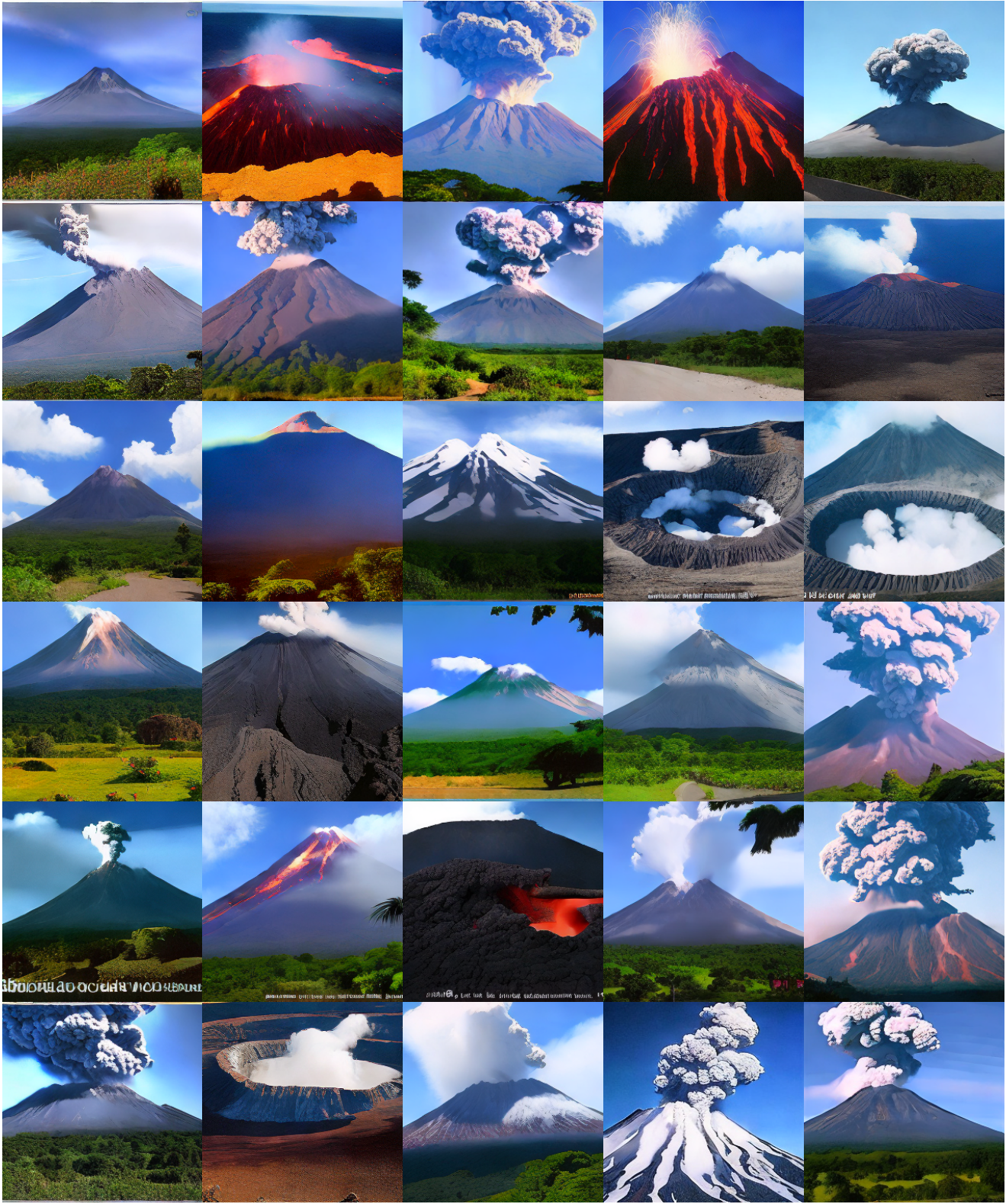

    \begin{center}
    \addimages{images/980/}
    \end{center}
    \caption{Uncurated 256² images for the class \emph{volcano} (980).}
    \label{fig:volcano}
\end{figure*}

\end{document}